\documentclass[3p,twocolumn]{elsarticle}

\usepackage[english]{babel}

\usepackage{amssymb,amsmath,amsfonts}
\usepackage{latexsym}
\usepackage{amsthm}
   
\usepackage{dsfont}
\usepackage{bm}
\usepackage{bbm}

\usepackage{graphicx}
\usepackage{subfigure}

\usepackage{enumitem}

\makeatletter
\def\ps@pprintTitle{%
	\let\@oddhead\@empty
	\let\@evenhead\@empty
	\def\@oddfoot{}%
	\let\@evenfoot\@oddfoot}
\makeatother

\newcommand{\R}{\mathbb{R}}

\newcommand{\Z}{\mathbb{Z}}
\newcommand{\dd}{{\rm d}}
\newcommand{\ReLU}{\operatorname{ReLU}}
\newcommand{\sigmoid}{\operatorname{sigmoid}}
\newcommand{\ExSpliNet}{\operatorname{ExSpliNet}}
\newcommand{\ONN}{\textit{\textbf{ONN}}}
\newcommand{\INN}[1]{\textit{\textbf{INN}}_{#1}}
\newcommand{\SciNot}[2]{#1 \cdot 10^{#2}}

\newtheorem{theorem}{Theorem}[section]

\newtheorem{proposition}[theorem]{Proposition}

\newtheorem{definition}[theorem]{Definition}

\newtheorem{remark}[theorem]{Remark}
\newtheorem{experiment}[theorem]{Experiment}

\begin{document}
	
	\title{ExSpliNet: An interpretable and expressive spline-based neural network}
	
	\author{Daniele Fakhoury}
	\author{Emanuele Fakhoury}
	\author{Hendrik Speleers}
	\address{University of Rome Tor Vergata, Rome, Italy}
	
	\begin{abstract}
		In this paper we present ExSpliNet, an interpretable and expressive neural network model. The model combines ideas of Kolmogorov neural networks, ensembles of probabilistic trees, and multivariate B-spline representations.
		We give a probabilistic interpretation of the model and show its universal approximation properties. We also discuss how it can be efficiently encoded by exploiting B-spline properties. Finally, we test the effectiveness of the proposed model on synthetic approximation problems and classical machine learning benchmark datasets.
	\end{abstract}
	
	\maketitle
	
	\section{Introduction}
	Solving problems that require the approximation of data in high-dimensional spaces is computationally extremely challenging. Most of the classical approximation methods suffer from the so-called \emph{curse of dimensionality} --- their complexity grows exponentially in the dimension --- and thus in practice they can only be applied to deal with lower-dimensional problems. On the other hand, machine learning techniques and in particular \emph{(deep) neural networks} are gaining in popularity as they have  shown outstanding performance in attacking different kinds of high-dimensional problems, especially in the context of image analysis and pattern recognition.
	
	\subsection{Interpretability of neural networks}\label{sec:interpretabibility}
	Neural networks are by nature incredibly complicated models with lots of parameters and it is in general difficult to interpret the behavior of the resulting functions in terms of those parameters.
	For this reason, neural networks are often called \emph{black-box models} --- the neural network functions might be highly unpredictable --- and this makes them less suited for highly risky tasks where one needs to understand why the machine takes a particular decision.	
	As a remedy, task-dependent assumptions can be imposed on the structure of the neural network (for example, convolutional neural networks for image related tasks) and a posteriori gradient methods \cite{Ribeiro,Selvaraju} can be used to try to understand the contribution of features in the final decision.
	
	Alternatively, such problem can be circumvented by relying on simpler models such as linear and additive models \cite{Hastie}. Additive models are roughly motivated by the so-called Kolmogorov superposition theorem (KST); see Section~\ref{sec:expressivity}.
	Other attractive and widely used choices are tree models \cite{Breiman}. They are highly interpretable since the predictions are based on a list of rules that corresponds to a hierarchical partition of the input space. On the other hand, classical trees are prone to overfitting since they are a piecewise constant regressor, and in combination with a greedy algorithm used to learn them, they tend to be \emph{unstable} --- a small change in the data can lead to a large change in the structure of the optimal decision tree. 
	Probabilistic/fuzzy trees are less subject to noisy data and are able to handle uncertainty in inexact contexts and domains.
	The performance can be improved by taking an ensemble of trees such as random forests. 
	Several neural network architectures have been designed that represent additive models \cite{Agarwal,Potts} and classical \cite{Balestriero1,Yang} or probabilistic/fuzzy \cite{Biau,Kontschieder,Wang} tree structures.
	
	\subsection{Expressivity of neural networks}\label{sec:expressivity}
	Another concern with neural networks is the expressivity of the model. Approximation theory for neural networks started with the investigation of shallow networks and the nonconstructive universal approximation theorems of Cybenko \cite{Cybenko} and Hornik et al. \cite{Hornik}.
 	More recently, in \cite{Costarelli} a constructive theory for approximating absolutely continuous functions by series of sigmoidal functions was developed and related to the expressivity of shallow networks.
	In the last few years, the attention has shifted to the approximation properties of deep ReLU networks; see 
	\cite{Bach,Cohen,Montanelli1,Poggio,Telgarsky,Yarotsky2} and references therein. In particular, an important theoretical problem is to determine why and when deep networks lessen or break the curse of dimensionality to achieve a given accuracy.
	
	Kolmogorov \cite{Kolmogorov} proved that any multivariate continuous function can be written as a sum of univariate continuous functions. More formally, setting $\bm{x} := [x_1,\ldots,x_D]\in[0, 1]^D$, the KST states that any continuous function $ f : [0,1]^D \rightarrow \R$ can be decomposed as
	\begin{equation}\label{eq:KST}
	f(\bm{x}) = \sum_{i=1}^{2D + 1} \Phi_i\Biggl( \sum_{d=1}^{D} \Psi_{i,d}(x_d) \Biggr),
	\end{equation}
	where the $\Psi$'s and $\Phi$'s are univariate continuous functions  on $[0,1]$, called \emph{inner} and \emph{outer functions}, respectively.
	Theoretical connections of the KST with neural networks started
	with the work of Hecht-Nielsen \cite{Hecht}. He interpreted the KST as a neural network, whose activation functions
	were the inner and outer functions. 
	Later, K\r{u}rkov\'a discussed the relevance of the KST \cite{Kurkova1} and provided a direct proof of the universal approximation theorem of multilayer neural networks based on the KST \cite{Kurkova2}.
	More recently, in \cite{Montanelli2}, the KST was applied to lessen the curse of dimensionality.
	Several implementations and constructive algorithms were proposed to generalize and add regularity to the inner and outer functions \cite{Braun,Koppen,Sprecher1,Sprecher2}. In the same spirit, a Kolmogorov spline network (based on cubic splines) was developed in \cite{Igelnik}. 
	
	\subsection{Splines in neural networks}
	The selection of activation functions in a neural network has a significant impact on the training process. There is, however, no obvious way to choose them because the ``optimal choice'' may depend on the specific task or problem to be solved. 
	Nowadays, ReLU activation functions (and variations) are the default choice in the zoo of activation functions for many types of neural networks.

	Univariate spline functions are a powerful tool in approximation theory \cite{Lyche,Sande1,Schumaker}. These are piecewise polynomials of a certain degree and global smoothness. In particular, it is known that maximally smooth splines have an eminent approximation behavior per degree of freedom \cite{Bressan,Sande2}. ReLU functions are a special instance of splines --- they are linear spline functions.
	More general (learnable) spline activation functions were studied in \cite{Campolucci,Friedman,Guarnieri,Vecci} and recently in \cite{Bohra,Scardapane} showing that their flexibility as activation functions allows for a reduction of the overall size of the network for a given accuracy. Hence, there is a trade-off between architecture complexity and activation function complexity.
	
	Univariate spline functions can be represented as linear combinations of so-called \emph{B-splines}, a set of locally supported basis functions that forms a nonnegative partition of unity. B-spline representations are an attractive choice for approximating univariate smooth functions since they can be compactly described by a small amount of parameters, yet each parameter has a local effect.
	Moreover, fast and stable algorithms are available for their computation \cite{deBoor,Lyche}. Multivariate extensions can be easily obtained by taking tensor products of B-splines.
	
	The advantage of B-splines in neural networks has been widely acknowledged. They are directly incorporated in the hidden layer of the so-called B-spline neural network (BSNN) \cite{Harris}. Thanks to the local support property of B-splines, such a network stores the information locally, which means that learning in one part of the input space minimally affects the rest of it, and this is very suitable for tasks of system identification; see, e.g., \cite{Coelho2,Lightbody,WangL}. Unfortunately, the BSNN is infeasible for data defined on high-dimensional domains due to the costly tensor-product structure \cite{Karagoz}.
	Convolutional neural networks, where the convolution operator is based on B-splines, have been proposed in \cite{Fey}. State-of-the-art results were achieved in the fields of image graph classification, shape correspondence and graph node classification, while being significantly faster. 
	
	Recently, Balestriero and Baraniuk \cite{Balestriero2} built a general bridge between  neural networks and spline approximation theory. They show that a large class of neural networks (including feed-forward ReLU networks and convolutional neural networks) can be regarded as an {additive linear spline model} on a very unstructured partition of the domain.
	This unstructuredness, however, complicates the analysis and its interpretability. On the other hand, forcing the structure of the partition to be orthogonal is a kind of regularization technique that leads to higher performance in terms of generalization ability.

	\subsection{Contributions of the paper}
	In this paper we present a new neural network model, called \emph{ExSpliNet}, that combines ideas of Kolmogorov neural networks,  ensembles of probabilistic trees, and multivariate B-spline representations.
	
	In the vein of \eqref{eq:KST}, ExSpliNet uses univariate splines as inner functions that feed $L$-variate tensor-product splines as outer functions, all of them represented in terms of B-splines. Here, $L$ is supposed to be not too high.
	The difference with a standard Kolmogorov model (like the one in \cite{Igelnik}) is that the outer functions are allowed to be multivariate functions instead of univariate functions (the latter is a special case where $L=1$).
	The new model is a feasible generalization of the BSNN model towards high-dimensional data. Specifically, for low input dimensions, the network parameters can be chosen so that ExSpliNet reproduces the output of a BSNN (in this case $L=D$). However, for high input dimensions, one can rely on the Kolmogorov-like structure to avoid the use of high-variate tensor-product B-splines and still maintain expressive power.
	
	The inner functions act as low-dimensional feature extractors. The outer functions can be regarded as probabilistic trees.
	This brings the proposed network model in strong connection with probabilistic regression tree models (like the one in \cite{Alkhoury}) that are more interpretable and robust to noise.
	The complete model can be efficiently evaluated thanks to the computational properties of B-splines. Moreover, it is explicitly differentiable when taking B-splines of degrees at least two. The inner and outer functions are customizable and could be tailored in different shapes to address various tasks by controlling the trade-off between interpretability and expressivity of the model.
	
	In other words, ExSpliNet is a customizable, interpretable and expressive model obtained by integrating the good approximation properties of spline functions with the interpretability of probabilistic tree models and the feature learning capability of neural network models. In addition, it is endowed with the computational properties of spline functions represented in terms of B-splines.
	
	Furthermore, we carry out a theoretical study of the universal approximation properties of ExSpliNet. Specifically, for both extreme cases $L=1$ and $L=D$, we show that ExSpliNet has the ability of a universal approximator. The main ingredients of the proof are the KST and classical approximation estimates for multivariate splines.
	
	Finally, we illustrate the suitability of the proposed model to address data-driven function approximation and to face differential problems, in the spirit of physics-informed neural networks (PINNs) \cite{Raissi}. We also show the general applicability of the model for classical machine learning tasks like image classification and regression.
	
	\subsection{Outline of the paper}
	The paper is structured as follows. In Section~\ref{sec:Splines} we review B-splines and some of their main properties.
	In Section~\ref{sec:Architecture} we present our novel network model and discuss basic implementation aspects. 
	Section~\ref{sec:Gradients} focuses on the model's use in optimization and how to explicitly compute the gradients with respect to different parameters.
	In Section~\ref{sec:Interpretation} we give an interpretation of the model in terms of feature extractors and probabilistic trees, which is illustrated by means of the classical Iris dataset in Section~\ref{sec:Interpretation-iris}.
	In Section~\ref{sec:Approximation} we describe two universal approximation results for ExSpliNet, the first based on the KST and the second on multivariate spline theory.
	Section~\ref{sec:Experiments} demonstrates the effectiveness of the model on synthetic approximation problems and classical machine learning benchmark datasets.
	Finally, in Section~\ref{sec:Conclusions} we end with some concluding remarks and ideas for future research.

	\section{Preliminaries on B-splines}\label{sec:Splines}
	In this section we provide the definition and main properties of the B-spline functions that play a major role in our neural network model. For the sake of simplicity, we focus on B-splines defined on uniform partitions. We refer the reader to \cite{deBoor,Lyche} for more details on B-splines.
	
	In order to construct B-splines we need the concept of knot sequence. A knot sequence $\bm{\xi}$ is a nondecreasing sequence of real numbers,
	$$\bm{\xi} := \{\xi_1 \leq \xi_2 \leq \cdots \leq \xi_{r} \}.$$
	The elements of $\bm{\xi}$ are called \emph{knots}. 
	Assuming integer values $r\geq p+2\geq 2$, on such sequence we can define $N:=r-p-1$ B-splines of degree $p$.	
	\begin{definition}\label{def:Bspline}
	Given a knot sequence $\bm{\xi}$, the $n$-th B-spline of degree $p\geq0$ is identically zero if $\xi_{n+p+1} = \xi_{n} $ and otherwise defined recursively by
	\begin{align*}
	&B_{\bm{\xi},p,n}(x) := \frac{x-\xi_n}{\xi_{n+p}- \xi_n} B_{\bm{\xi},p-1,n}(x) \\
	&\quad+\frac{\xi_{n+p+1} -x}{\xi_{n+p+1} -\xi_{n+1}} B_{\bm{\xi},p-1,n+1}(x),
	\end{align*}
	starting from
	$$ B_{\bm{\xi},0,n} := \begin{cases}
	1,  \quad x\in[\xi_n,\xi_{n+1}),\\
	0,  \quad \mbox{otherwise}.
	\end{cases} $$
	Here, we use the convention that fractions with zero denominator have value zero.
	\end{definition}
	
	B-splines possess several interesting properties. The function 
	$B_{\bm{\xi},p,n}$ is a nonnegative, piecewise polynomial of degree $p$ that is locally supported on the interval $[\xi_{n},\xi_{n+p+1}]$.
	Moreover, its integral is equal to
	$$\int_{\xi_{n}}^{\xi_{n+p+1}}B_{\bm{\xi},p,n}(x)\,\dd x=\frac{\xi_{n+p+1}-\xi_n}{p+1},$$
	and for $p\geq1$ its right-hand derivative can be simply computed as
	\begin{equation}\label{eq:diff-Bspline}
	\frac{\dd }{\dd x^+}B_{\bm{\xi},p,n}(x)=p\biggl(\frac{B_{\bm{\xi},p-1,n}(x)}{\xi_{n+p}-\xi_n}-\frac{B_{\bm{\xi},p-1,n+1}(x)}{\xi_{n+p+1}-\xi_{n+1}}\biggr).
	\end{equation}
	
	From now on, for the sake of simplicity, we assume that the knots are chosen as
	\begin{equation}\label{eq:knots}
	\begin{aligned}
	  &\xi_1=\cdots=\xi_{p+1}=0,\\
	  &\xi_{p+i+1}=\frac{i}{N-p},\quad i=0,\ldots,N-p,\\
	  &\xi_{N+1}=\cdots=\xi_{N+p+1}=1,
	\end{aligned}
	\end{equation}
	for some integer values $N>p\geq0$.
	The corresponding B-splines are denoted by
	$$B_{N,p,n}(x):=B_{\bm{\xi},p,n}(x),\quad n=1,\ldots,N,$$
	and to avoid asymmetry at $x=1$, we define them to be left continuous there, i.e.,
	$$B_{N,p,n}(1):=\lim_{\substack{x\to1\\ x<1}}B_{\bm{\xi},p,n}(x), \quad n=1,\ldots,N.$$
	We collect these functions in the vector
	\begin{equation}\label{eq:vector-B-splines}
	\mathcal{B}_{N,p}(x) := [B_{N,p,1}(x),\ldots,B_{N,p,N}(x)].	  
	\end{equation}
	The B-splines in \eqref{eq:vector-B-splines} are linearly independent and belong to the continuity class $C^{p-1}([0,1])$. They span the full space of piecewise polynomials of degree less than or equal to $p$ that belong to $C^{p-1}([0,1])$  on the partition induced by the knots in \eqref{eq:knots}. Moreover, they sum up to one on $[0,1]$, i.e.,
	\begin{equation*}
		\sum_{n = 1}^N B_{N,p,n}(x) = 1, \quad x \in [0,1].
	\end{equation*}
	Note that the right-hand derivative in \eqref{eq:diff-Bspline} implies the standard derivative for $p\geq2$. 
	Some B-spline vectors are depicted in Figure~\ref{fig:Cardinal}.
	\begin{figure}[t!]
		\centering
		\subfigure[$p=1$]{\includegraphics[height=3.6cm]{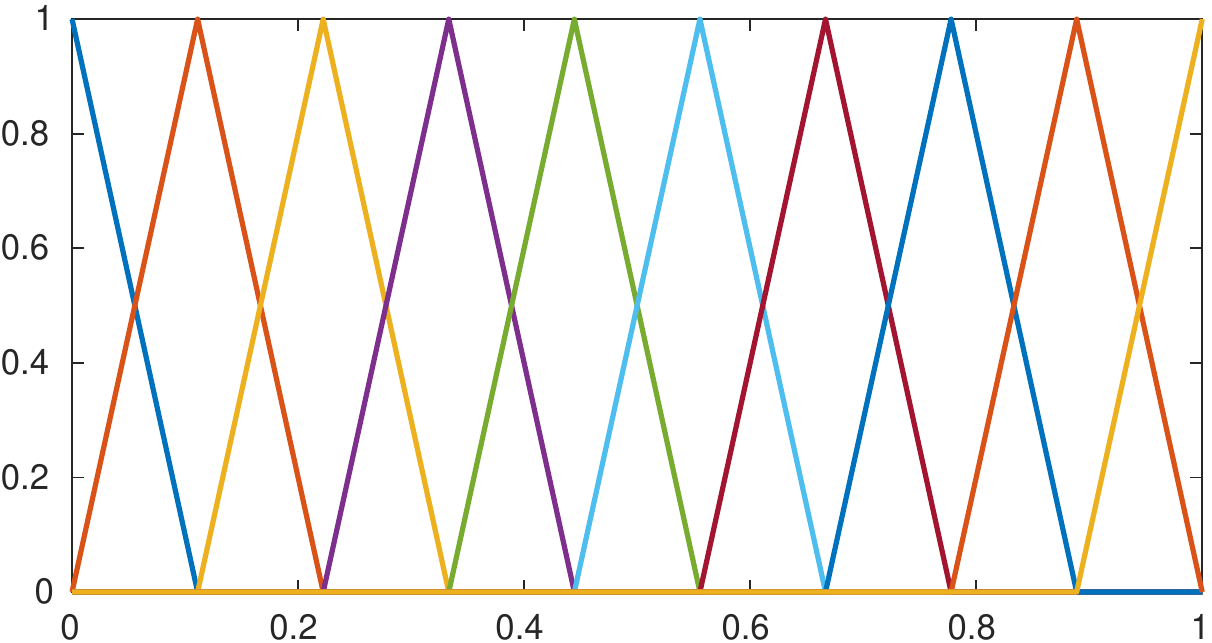}}
		\subfigure[$p=2$]{\includegraphics[height=3.6cm]{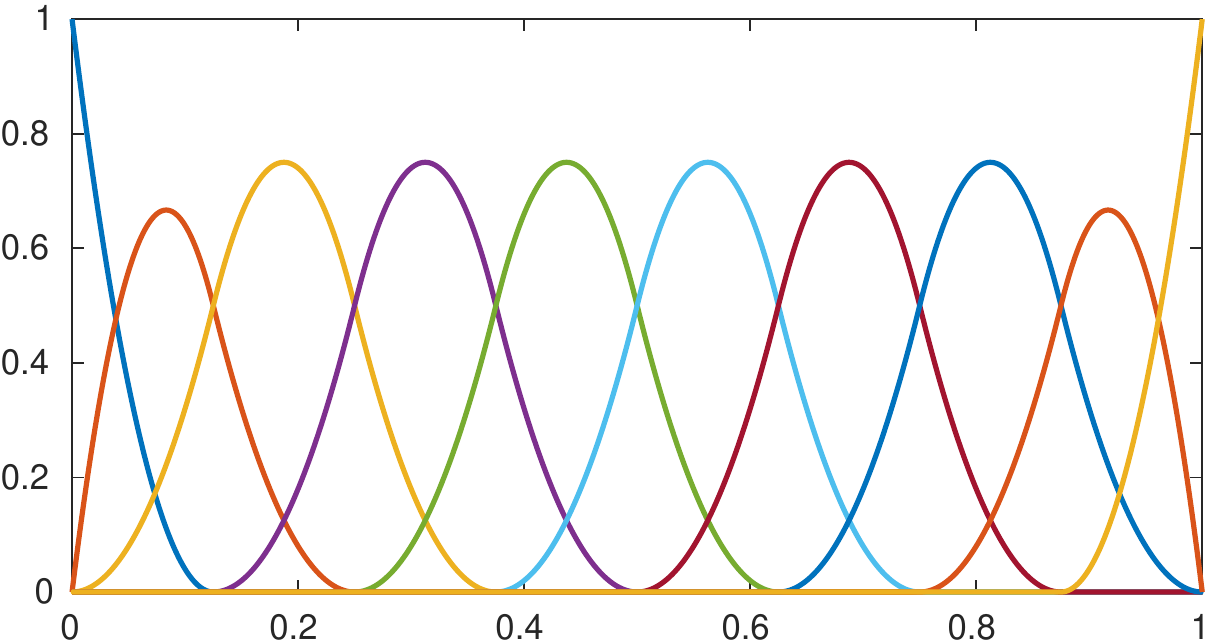}}
		\subfigure[$p=3$]{\includegraphics[height=3.6cm]{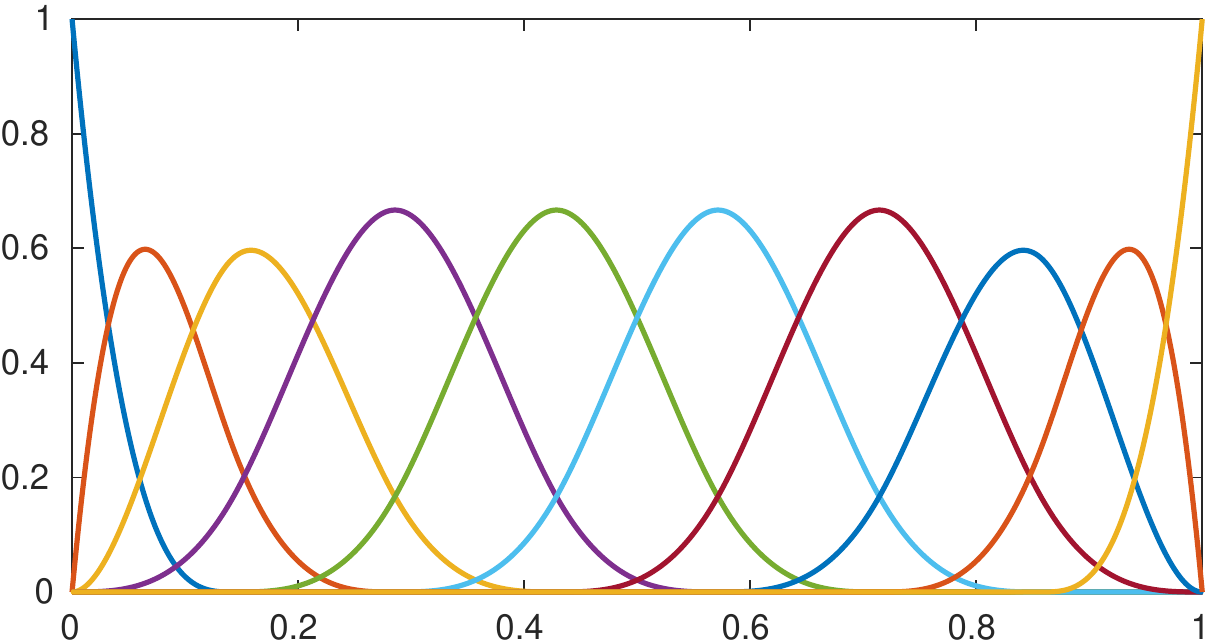}}
		\caption{The B-splines in the vectors $\mathcal{B}_{N,p}(x)$ for $N=10$, $p = 1,2,3$, and $x\in[0,1]$.}
		\label{fig:Cardinal}
	\end{figure}
	
	For $p\geq1$ the identity function on $[0,1]$ can be exactly represented as
	\begin{equation}\label{eq:identity-Bspline}
	\sum_{n = 1}^N \xi^*_{N,p,n} B_{N,p,n}(x) = x, \quad x \in [0,1],
	\end{equation}
	where 
	\begin{equation}\label{eq:greville}
	\xi^*_{N,p,n}:=\frac{\xi_{n+1}+\ldots+\xi_{n+p}}{p}, \quad n=1,\ldots,N,  
	\end{equation}
	are the so-called \emph{Greville abscissae}. We observe that $0=\xi^*_{N,p,1}<\xi^*_{N,p,n}< \xi^*_{N,p,N}=1$ for $1<n<N$.
	
	A \textit{spline} is a linear combination of B-splines, say
	\begin{equation}\label{eq:spline}
	s(x) := \sum_{n = 1}^N w_n B_{N,p,n}(x), \quad x \in [0,1],
	\end{equation}
	for given weights $w_n\in\R$
	and can be efficiently evaluated via the recursive algorithm described in Proposition~\ref{pro:deBoor}. 
	This evaluation procedure is known as the \emph{de Boor algorithm}.
	Note that, due to the local support property, at most $p+1$ consecutive B-splines (instead of $N$) are nonzero at any $x$.
	
	\begin{proposition}\label{pro:deBoor}
	Let $s$ be a spline represented as in \eqref{eq:spline}.
	Assume $x\in[\xi_{m},\xi_{m+1})$ such that $p+1\leq m\leq N$. Set $w_{n,0}:=w_{n}$ for $n=m-p,\ldots,m$, and 
	$$w_{n,q+1} := \frac{x-\xi_n}{\xi_{n+p-q}-\xi_n}w_{n,q} + \frac{\xi_{n+p-q}-x}{\xi_{n+p-q}-\xi_n}w_{n-1,q},$$
	for $n=m-p+q+1,\ldots,m$ and $q=0,\ldots,p-1$.
	Then, we have $s(x)=w_{m,p}$.
	\end{proposition}
	
	Given a vector $\bm{w}:=[w_1,\ldots,w_N]\in\R^N$, we can compactly write any linear combination of the functions in $\mathcal{B}_{N,p}(x)$ via a dot product, i.e.,
	\begin{equation}\label{eq:spline-vector}
	s^{N,p}_{\bm{w}}(x) := \bm{w}\cdot \mathcal{B}_{N,p}(x).
	\end{equation}
		
	Finally, we extend the above functions to the multivariate setting by using a tensor-product structure.	
	Given the vectors $\bm{p} := [p_1, \ldots, p_D]\in\Z^D$ with each $p_d\geq0$, $\bm{N} := [N_1, \ldots, N_D]\in\Z^D$ with each $N_d>p_d$,  and $\bm{x} := [x_1, \ldots, x_D] \in [0,1]^D$, 
	we define the vector of tensor-product B-splines as
	\begin{equation}\label{eq:vector-B-splines-TP}
	\mathcal{B}_{\bm{N},\bm{p}}(\bm{x}):= \bigotimes_{d=1}^D \mathcal{B}_{N_d,p_d}(x_d).
	\end{equation}
	Similar to \eqref{eq:spline-vector}, given a vector $\bm{w}\in\R^{N_1\cdots N_D}$, we can compactly write any linear combination as
	$$s^{\bm{N},\bm{p}}_{\bm{w}}(\bm{x}) := \bm{w} \cdot \mathcal{B}_{\bm{N},\bm{p}}(\bm{x}).$$
	Note that, by exploiting the inherent tensor-product structure, most mathematical operations (like evaluation) on the $D$-variate spline $s^{\bm{N},\bm{p}}_{\bm{w}}$ can be simply transformed into a sequence of analogous operations on univariate splines of the form \eqref{eq:spline-vector}.

	\section{ExSpliNet} \label{sec:Architecture}
	In this section we describe the general architecture of our Kolmogorov-like neural network model, called \emph{ExSpliNet}. We also discuss some implementation aspects.
	
	\subsection{Network architecture}
	For a given input dimension $D$ and output dimension $O$, we fix two additional integer hyperparameters $T$ and $L$. We will refer to $T$ as the \emph{number of trees} and $L$ as the \emph{number of levels} (this terminology will be clarified in Section~\ref{sec:Interpretation-outer}).
	Then, we specify the vectors $\bm{p} := [p_1,\ldots,p_L]\in\Z^L$ and $\bm{N} := [N_1,\ldots,N_L]\in\Z^L$ such that each $N_\ell> p_\ell\geq 0$, and the vectors $\bm{q} := [q_1,\ldots,q_L]\in\Z^L$ and $\bm{M} := [M_1,\ldots,M_L]\in\Z^L$ such that each $M_\ell> q_\ell\geq 0$.
	
	For $t = 1,\ldots,T$, $\ell=1,\ldots,L$, $d=1,\ldots,D$, and $o=1,\ldots,O$, the weight parameters of the network model are given by the vectors $\bm{v}^{t,\ell,d} \in \R^{N_\ell} $ and $\bm{w}^{o,t} \in \R^{M_1 \cdots M_L}$.
	We group all those weight parameters into
	$$ \bm{V} :=[(\bm{v}^{t,\ell,d})_{t,\ell,d}], \quad \bm{W} :=[(\bm{w}^{o,t})_{o,t}],$$
	and call them the \emph{inner} and \emph{outer weights}, respectively. We also define the subgroups $\bm{V}^{t,\ell}:=[\bm{v}^{t,\ell,1},\ldots,\bm{v}^{t,\ell,D}]$ and
	$\bm{W}^o :=[\bm{w}^{o,1},\ldots,\bm{w}^{o,T}]$.
	
	We are now ready to define our Kolmogorov-like neural network model with spline inner and outer functions. We assume that the input variables are given by $\bm{x}:=[x_1,\ldots,x_D]\in[0,1]^D$.
	\begin{definition}\label{def:ExSpliNet}
	The model $\ExSpliNet$ is a family of Kolmogorov-like functions  parameterized by $\bm{V}$ and $\bm{W}$ as
	\begin{gather*}
	\ExSpliNet^{T,\bm{N},\bm{M},\bm{p},\bm{q}}_{\bm{V},\bm{W}}: [0,1]^D \to \R^O:\\
	\bm{x}  \to
	 \Bigl[ E^{T,\bm{N},\bm{M},\bm{p},\bm{q}}_{\bm{V},\bm{W}^1} (\bm{x}), \ldots, E^{T,\bm{N},\bm{M},\bm{p},\bm{q}}_{\bm{V},\bm{W}^O}(\bm{x} ) \Bigr],
	\end{gather*}
	where
   $$E^{T,\bm{N},\bm{M},\bm{p},\bm{q}}_{\bm{V},\bm{W}^{o}} (\bm{x}) := 
	 \sum_{t=1}^{T} \Phi^{\bm{M},\bm{q}}_{\bm{w}^{o,t}}\Biggl( \biggl[\sum_{d=1}^D \Psi^{N_\ell,p_\ell}_{\bm{v}^{t,\ell,d}}(x_d) \biggr]_{\ell} \Biggr)$$
	and
	\begin{align*}
	\Psi^{N_\ell,p_\ell}_{\bm{v}^{t,\ell,d}}(x_d) &:= \bm{v}^{t,\ell,d}\cdot \mathcal{B}_{N_\ell,p_\ell}(x_d), \quad x_d\in[0,1],\\
	\Phi^{\bm{M},\bm{q}}_{\bm{w}^{o,t}}(\bm{y}_t) &:= \bm{w}^{o,t}\cdot \mathcal{B}_{\bm{M},\bm{q}}(\bm{y}_t), \quad \bm{y}_t\in[0,1]^L.
	\end{align*}
	\end{definition}
	
	Each Kolmogorov-like function $E^{T,\bm{N},\bm{M},\bm{p},\bm{q}}_{\bm{V},\bm{W}^{o}}$ has univariate splines as inner functions and $L$-variate splines as outer functions. 
	To be sure that this function is well defined, it is required that the range of each function
	\begin{equation}\label{eq:decision-fun}
	\Psi^{N_\ell,p_\ell}_{\bm{V}^{t,\ell}}(\bm{x}) := \sum_{d=1}^D \Psi^{N_\ell,p_\ell}_{\bm{v}^{t,\ell,d}}(x_d)
	\end{equation} 
	belongs to the interval $[0,1]$ since it is an argument of the outer spline function $\Phi^{\bm{M},\bm{q}}_{\bm{w}^{o,t}}$. 
	As shown in Proposition~\ref{pro:decision-fun-range}, this can be safeguarded by imposing that the components $v^{t,\ell,d}_{n_\ell}$, $n_\ell=1,\ldots,N_\ell$ of the vector $\bm{v}^{t,\ell,d} \in \R^{N_\ell} $ satisfy 
    \begin{equation}\label{eq:V-condition}
    0\leq v^{t,\ell,d}_{n_\ell} \leq \nu^{t,\ell,d}, \quad \sum_{d=1}^{D}\nu^{t,\ell,d}\leq1,
    \end{equation}
    for some $\nu^{t,\ell,d}\in\R$. From now on we assume that a valid instance of the network model satisfies \eqref{eq:V-condition}.
	
	\begin{proposition}\label{pro:decision-fun-range}
	The function in \eqref{eq:decision-fun} satisfies
	\begin{equation}\label{eq:decision-fun-range}
	0\leq \Psi^{N_\ell,p_\ell}_{\bm{V}^{t,\ell}}(\bm{x})\leq1,\quad \bm{x} \in [0,1]^D, 
	\end{equation}
	under the assumption \eqref{eq:V-condition}.
	\end{proposition}
	\begin{proof}
	Fix $\bm{x} \in [0,1]^D$. By construction we have
	$$\Psi^{N_\ell,p_\ell}_{\bm{V}^{t,\ell}}(\bm{x})
	= \sum_{d=1}^D  \bm{v}^{t,\ell,d} \cdot \mathcal{B}_{N_\ell,p_\ell}(x_d).$$
	Moreover, by the nonnegativity and the partition-of-unity property of the functions collected in the vector $\mathcal{B}_{N_\ell,p_\ell}(x_d)$, it is clear that 
	$$\min(\bm{v}^{t,\ell,d}) \leq \bm{v}^{t,\ell,d} \cdot \mathcal{B}_{N_\ell,p_\ell}(x_d) \leq \max(\bm{v}^{t,\ell,d}).$$
	Thus,
	$$\sum_{d=1}^D\min(\bm{v}^{t,\ell,d}) \leq \Psi^{N_\ell,p_\ell}_{\bm{V}^{t,\ell}}(\bm{x}) \leq \sum_{d=1}^D \max(\bm{v}^{t,\ell,d}).$$
	These inequalities combined with the conditions in \eqref{eq:V-condition} ensure that \eqref{eq:decision-fun-range} is satisfied.
	\end{proof}
	
	\begin{remark}\label{rmk:V-condition-simple}
	It is easy to see that the conditions in \eqref{eq:V-condition} are satisfied when imposing the simpler conditions
	\begin{equation*}%\label{eq:V-condition-simple}
		0\leq v^{t,\ell,d}_{n_\ell}, \quad \sum_{d=1}^{D} \sum_{n_\ell=1}^{N_\ell} v^{t,\ell,d}_{n_\ell} \leq 1,
	\end{equation*}
    and thus they also ensure \eqref{eq:decision-fun-range}.
	\end{remark}
	
	\begin{figure}[h!]
		\centering
		\subfigure[Inner neural network]{\includegraphics[scale=0.19]{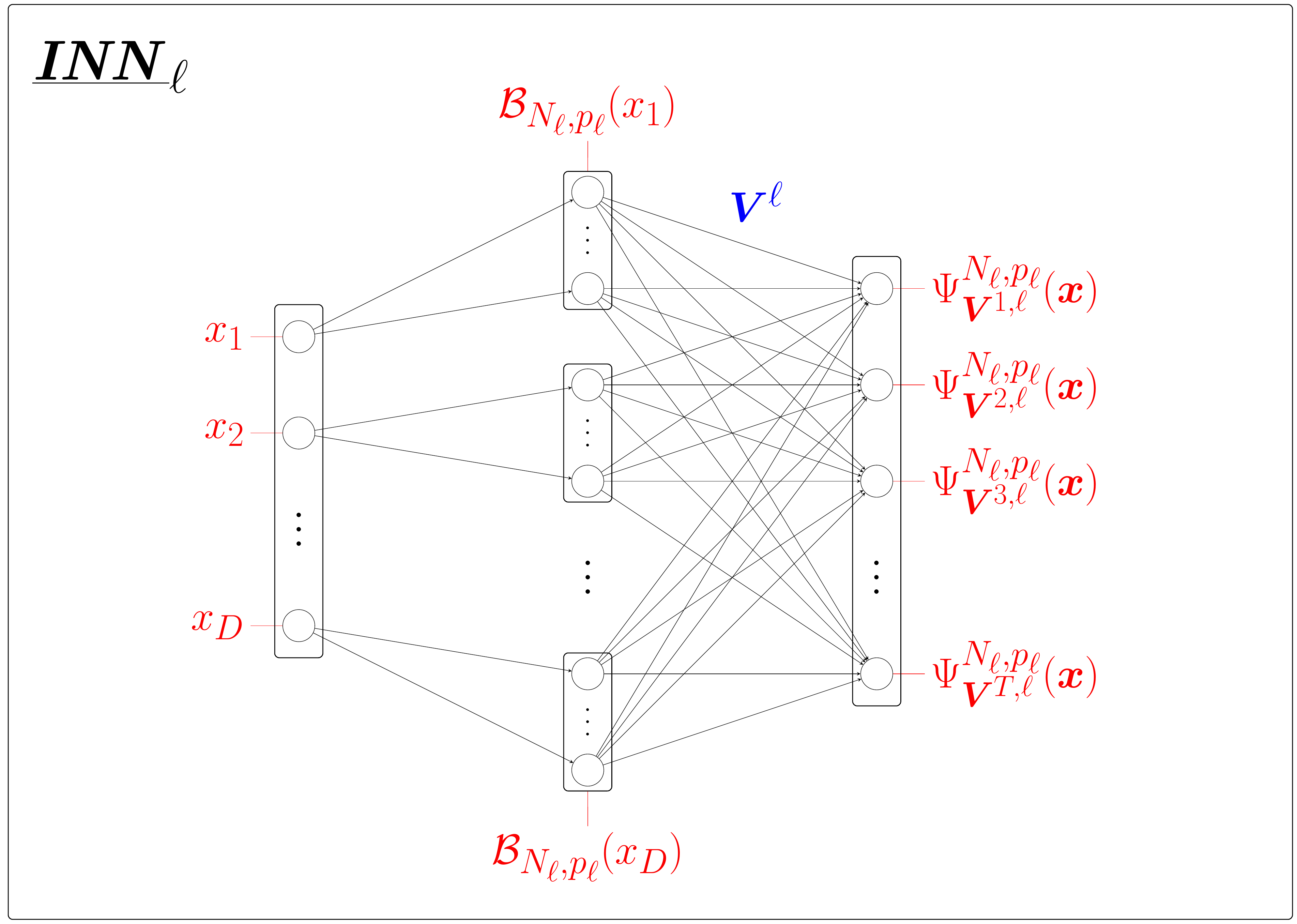}}
		\subfigure[Outer neural network]{\includegraphics[scale=0.19]{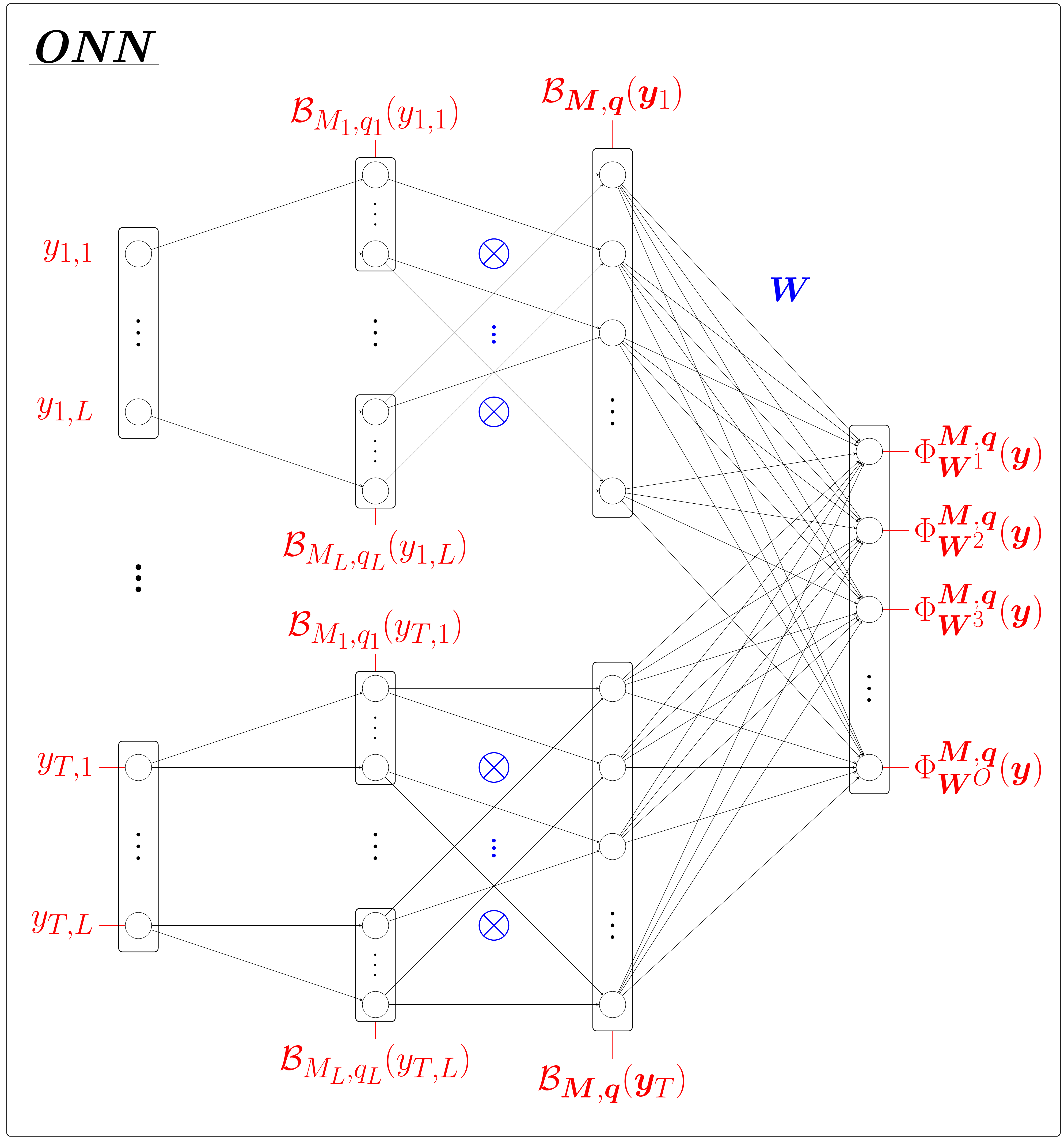}}
		\subfigure[ExSpliNet]{\includegraphics[scale=0.19]{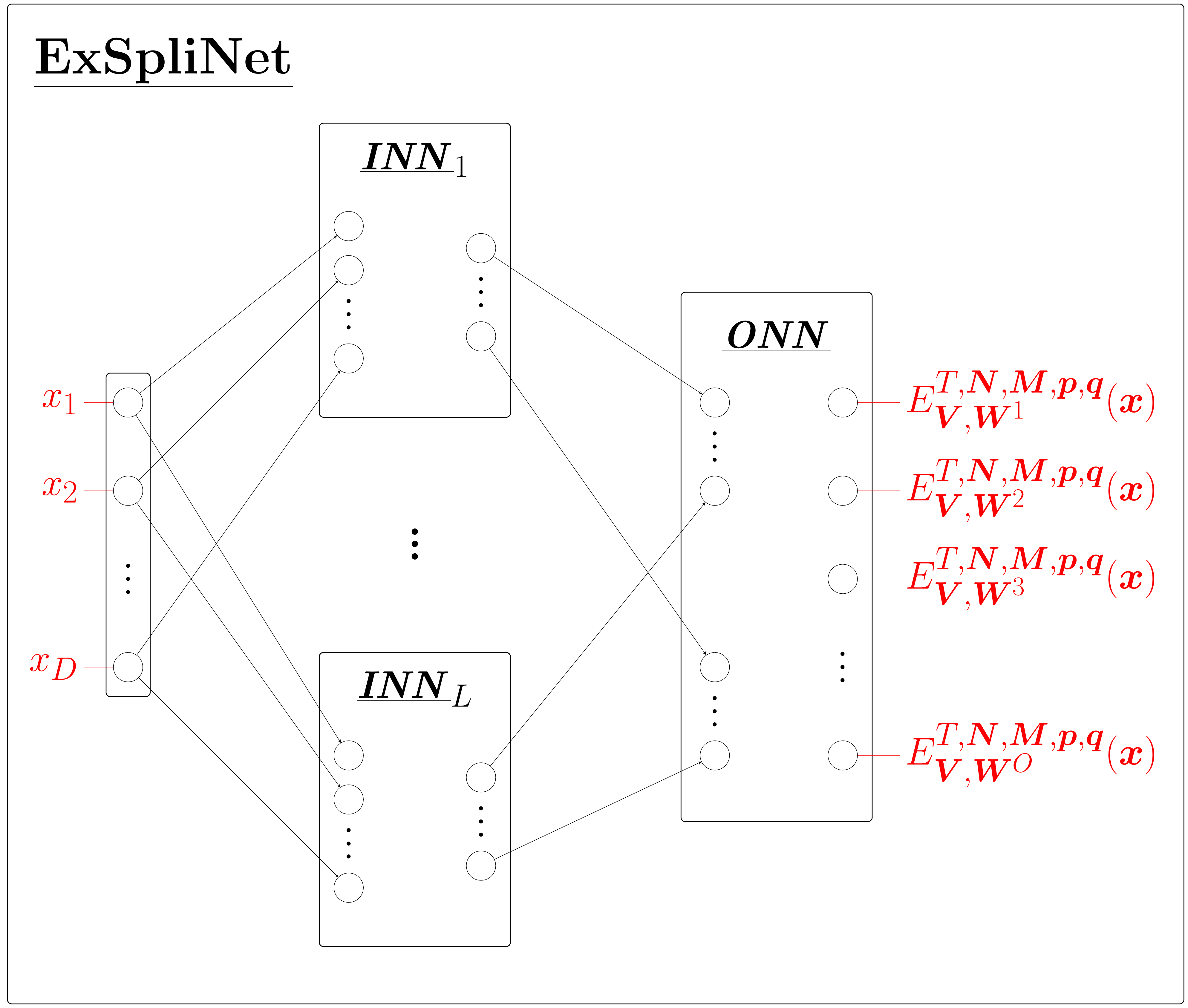}}
		\caption{A graphical representation of the architecture of ExSpliNet as a composition of inner and outer neural networks.}
		\label{fig:architecture}\vspace*{-0.8cm}
	\end{figure}
	
	The complexity of the network model can be described as follows.
	\begin{proposition}\label{pro:params}
	The number of weight parameters involved in $\ExSpliNet^{T,\bm{N},\bm{M},\bm{p},\bm{q}}_{\bm{V},\bm{W}}$ is
	$$ D T \sum_{\ell=1}^{L} N_\ell + O T \prod_{\ell=1}^L M_\ell.$$ 
	\end{proposition}
	\begin{proof}
	  This can be verified by a direct counting of the weight parameters.
	\end{proof}
	
	We provide a graphical representation of the network model in Figure~\ref{fig:architecture}. 
	ExSpliNet can be visualized as a composition of inner and outer neural networks.
	Given $\bm{x}:=[x_1,\ldots,x_D]\in[0,1]^D$, the inner neural network model $\INN{\ell}$ computes the functions $\Psi^{N_\ell,p_\ell}_{\bm{V}^{t,\ell}}(\bm{x})$ in \eqref{eq:decision-fun} for $t=1,\ldots,T$. The hidden layer consists of B-spline activation functions to obtain the vectors $\mathcal{B}_{N_\ell,p_\ell}(x_d)$ for $d=1,\ldots,D$. The output is then achieved as a linear combination of them using the weights $\bm{V}^{\ell}:=[\bm{V}^{1,\ell},\ldots,\bm{V}^{T,\ell}]$. There are in total $L$ inner neural networks ($\ell=1,\ldots,L$).
	Given $\bm{y}_t:=[y_{t,1},\ldots,y_{t,L}]\in[0,1]^L$ for $t=1,\ldots,T$, the outer neural network model $\ONN$ computes the functions
	$$\Phi^{\bm{M},\bm{q}}_{\bm{W}^{o}}(\bm{y}) := \sum_{t=1}^{T} \Phi^{\bm{M},\bm{q}}_{\bm{w}^{o,t}}(\bm{y}_t),$$
	for $o=1,\ldots,O$, where $\bm{y}:=[\bm{y}_1,\ldots,\bm{y}_T]$. The first hidden layer consists again of B-spline activation functions and computes the vectors $\mathcal{B}_{M_\ell,q_\ell}(y_{t,\ell})$ for $t=1,\ldots,T$ and $\ell=1,\ldots,L$. The next hidden layer blends them via tensor products into the vectors $\mathcal{B}_{\bm{M},\bm{q}}(\bm{y}_t)$ for $t=1,\ldots,T$. Finally, the output is achieved as a linear combination of them using the weights $\bm{W}:=[\bm{W}^{1},\ldots,\bm{W}^{O}]$. ExSpliNet can be simply obtained by putting a layer of inner neural networks in front of the outer neural network.
	
	\subsection{Efficient evaluation and natural dropout} \label{sec:Evaluation}
	From Definition~\ref{def:ExSpliNet} it is clear that any ExSpliNet function $E^{T,\bm{N},\bm{M},\bm{p},\bm{q}}_{\bm{V},\bm{W}^{o}}$ is a composition of univariate and multivariate spline functions. 
	
	Thanks to their representation in terms of B-splines, univariate splines can be efficiently evaluated via the de Boor algorithm (see Proposition~\ref{pro:deBoor}).
	A key ingredient is the local support of B-splines, so that many terms are zero in a sum like \eqref{eq:spline}. Specifically, for a given real value $x\in[0,1]$ and $p\geq1$, the indices of the nonzero components of the univariate B-spline vector $\mathcal{B}_{N,p}(x)$ in \eqref{eq:vector-B-splines} are part of the following subset of indices:
	\begin{equation*}%\label{eq:shifts}
		\lfloor \zeta_{N,p}(x) \rfloor + 1 - p, \ldots, \lceil \zeta_{N,p}(x) \rceil,
	\end{equation*}
	where 
	$$\zeta_{N,p}(x):=(1- x)p + x N.$$
	The above reasoning can be easily extended to the multivariate B-spline vector $\mathcal{B}_{\bm{N},\bm{p}}(\bm{x})$ in \eqref{eq:vector-B-splines-TP} thanks to the inherent tensor-product structure. 
	
	Several vectors of type $\mathcal{B}_{N,p}(x)$ and of type $\mathcal{B}_{\bm{N},\bm{p}}(\bm{x})$ are involved in our network model.
	The precise knowledge of the locality of the supports of the B-splines may be exploited in an efficient implementation of the network.
	This can be regarded as a kind of ``natural dropout'' in the ExSpliNet model.
	
	Dropout is a standard technique to reduce the computational complexity of deep neural networks \cite{Srivastava}. However, randomly dropping out nodes during training seems artificial and has the drawback that additional hyperparameters (for example, the probability of a neuron being dropped out) need to be chosen and validated. The natural dropout in ExSpliNet does not suffer from such issues.

	\section{Explicit expressions for gradients} \label{sec:Gradients}
	We now briefly discuss the explicit computation of the gradient of an ExSpliNet function with respect to the input variables as well as the weight parameters. The simple differentiation formula for B-splines, see \eqref{eq:diff-Bspline}, allows us to derive explicit expressions for them.
	For ease of presentation but without loss of generality, we just consider $O = 1$ so that we can drop all superscripts $o$ in our notations. 
    All splines involved in our network model are differentiable for degrees at least two. 
	Therefore, in this section, we assume all $p_\ell\geq2$ and $q_\ell\geq2$. Note that for lower degrees we can still compute the right-hand derivatives according to \eqref{eq:diff-Bspline}. 
	
	Before going into the details of the gradient computation, let us first collect the functions of type \eqref{eq:decision-fun} for some $t$ into the vector
	\begin{equation*}%\label{eq:decision-vec}
	\Psi^{\bm{N},\bm{p}}_{\bm{V}^t} (\bm{x}) := \Bigl[\Psi^{N_1,p_1}_{\bm{V}^{t,1}}(\bm{x}), \ldots, \Psi^{N_L,p_L}_{\bm{V}^{t,L}}(\bm{x}) \Bigr].
	\end{equation*}
	Then, we can compactly write 
	$$E^{T,\bm{N},\bm{M},\bm{p},\bm{q}}_{\bm{V},\bm{W}} (\bm{x}) = \sum_{t=1}^{T}  \Phi^{\bm{M},\bm{q}}_{\bm{w}^{t}}\bigl(\Psi^{\bm{N},\bm{p}}_{\bm{V}^t}(\bm{x}) \bigr).$$
	We set $\bm{x} := [x_1, \ldots, x_D]$ and $\bm{y}_t:=[y_{t,1},\ldots,y_{t,L}]$ in the following.
	
	\subsection{Gradient with respect to the input variables}
	From Definition~\ref{def:ExSpliNet} and the differentiation formula in \eqref{eq:diff-Bspline} we immediately deduce that
	$$\frac{\dd }{\dd x_d} \Psi^{N_\ell,p_\ell}_{\bm{v}^{t,\ell,d}}(x_d) 
	%= \bar{\bm{v}}^{t,l,d}\cdot \mathcal{B}_{N_d-1, p_\ell-1}(x_d)
	=\Psi^{N_\ell-1,p_\ell-1}_{\bar{\bm{v}}^{t,\ell,d}}(x_d),$$
	for some values of $\bar{\bm{v}}^{t,\ell,d}$. Similarly, given $1\leq \ell\leq L$, we get
	$$\frac{\partial}{\partial y_{t,\ell}}\Phi^{\bm{M},\bm{q}}_{\bm{w}^{t}}(\bm{y}_t)
	=\Phi^{\bm{M}[\ell-],\bm{q}[\ell-]}_{\bar{\bm{w}}^{t}}(\bm{y}_t),$$
	for some values of $\bar{\bm{w}}^{t}$ and
	\begin{align*}
	\bm{M}[\ell-]&:=[M_1,\ldots,M_{\ell-1},M_\ell-1,M_{\ell+1},\ldots,M_L],\\
	\bm{q}[\ell-]&:=[q_1,\ldots,q_{\ell-1},q_\ell-1,q_{\ell+1},\ldots,q_L].
	\end{align*}
	Then, by applying the chain rule for derivatives, we can write the partial derivative of an ExSpliNet function with respect to the $d$-th input variable as
	\begin{align*}
	&\frac{\partial}{\partial x_d}E^{T,\bm{N},\bm{M},\bm{p},\bm{q}}_{\bm{V},\bm{W}} (\bm{x})\\
	&= \sum_{t=1}^{T}\sum_{\ell=1}^L \frac{\partial}{\partial y_{t,\ell}}\Phi^{\bm{M},\bm{q}}_{\bm{w}^{t}}\bigl( \Psi^{\bm{N},\bm{p}}_{\bm{V}^t} (\bm{x}) \bigr)\frac{\partial}{\partial x_d}\Psi^{N_\ell,p_\ell}_{\bm{V}^{t,\ell}}(\bm{x}) \\
	&= \sum_{t=1}^{T}\sum_{\ell=1}^L \Phi^{\bm{M}[\ell-],\bm{q}[\ell-]}_{\bar{\bm{w}}^{t}}\bigl( \Psi^{\bm{N},\bm{p}}_{\bm{V}^t} (\bm{x}) \bigr)
	\Psi^{N_\ell-1,p_\ell-1}_{\bar{\bm{v}}^{t,\ell,d}}(x_d).
	\end{align*}
	Since this expression is a composition of B-splines (actually very similar to an ExSpliNet function), only few of them are involved in the evaluation at a given $\bm{x}$, as explained in Section~\ref{sec:Evaluation}. Therefore, evaluation of a partial derivative of an ExSpliNet function is not expensive. It is clear that an analogous reasoning also holds for higher-order partial derivatives.
	
	Such knowledge is valuable, for example, in the implementation of PINNs for the solution of differential problems \cite{Lu,Raissi}. It avoids the need for using an automatic differentiation method.
	
	\subsection{Gradient with respect to the weight parameters}
	We now look at the derivatives with respect to the weight parameters. There are two types of such parameters: the inner and outer weights.
	Let us first consider the outer weights. Given $1\leq m\leq M_1\cdots M_L$ and $1\leq t\leq T$, it is clear that
	$$\frac{\partial}{\partial w^{t}_m}E^{T,\bm{N},\bm{M},\bm{p},\bm{q}}_{\bm{V},\bm{W}} (\bm{x}) = \Phi^{\bm{M},\bm{q}}_{\bm{e}^{m}}\bigl( \Psi^{\bm{N},\bm{p}}_{\bm{V}^t} (\bm{x}) \bigr),$$
	where $\bm{e}^{m}$ is the unit vector of length $M_1\cdots M_L$ with the value $1$ at the $m$-th position and $0$ elsewhere.
	
	The inner weights can be addressed as follows. Given $1\leq n_\ell\leq N_\ell$, $1\leq t\leq T$, $1\leq \ell\leq L$, and $1\leq d\leq D$, we deduce
	\begin{align*}
	&\frac{\partial}{\partial v^{t,\ell,d}_{n_\ell}}E^{T,\bm{N},\bm{M},\bm{p},\bm{q}}_{\bm{V},\bm{W}} (\bm{x}) \\
	&= \frac{\partial}{\partial y_{t,\ell}}\Phi^{\bm{M},\bm{q}}_{\bm{w}^{t}}\bigl( \Psi^{\bm{N},\bm{p}}_{\bm{V}^t} (\bm{x}) \bigr)\frac{\partial}{\partial v^{t,\ell,d}_{n_\ell}}\Psi^{N_\ell,p_\ell}_{\bm{V}^{t,\ell}}(\bm{x}) \\
	&= \Phi^{\bm{M}[\ell-],\bm{q}[\ell-]}_{\bar{\bm{w}}^{t}}\bigl( \Psi^{\bm{N},\bm{p}}_{\bm{V}^t} (\bm{x}) \bigr)
	\Psi^{N_\ell,p_\ell}_{\bm{e}^{n_\ell}}(x_d),
	\end{align*}
	where $\bm{e}^{n_\ell}$ is the unit vector of length $N_\ell$ with the value $1$ at the $n_\ell$-th position and $0$ elsewhere.
	
	Finally, we discuss the computation of the derivative of a loss function with respect to the weight parameters of ExSpliNet in a supervised learning environment. 
	Let $A := \{ (\bm{x}^1, y^1), \ldots, (\bm{x}^K, y^K) \}$ be a training dataset
	where $\bm{x}^k \in [0,1]^D$ and $y^k \in \R$.
	The empirical risk function over the training data is defined by
	\begin{equation} \label{eq:risk}
	\mathcal{E}(A) := \frac{1}{K} \sum_{k=1}^K \mathcal{F}\bigl(E^{T,\bm{N},\bm{M},\bm{p},\bm{q}}_{\bm{V},\bm{W}} (\bm{x}^k), y^k\bigr), 
	\end{equation}
	for a given loss function $\mathcal{F}(z,y)$. 
	For simplicity, in the following we consider the squared loss function $\mathcal{F}(z,y) := (z-y)^2$. For empirical risk minimization it is convenient to be able to compute the gradient with respect to the weight parameters.
	Let $w$ be any weight parameter, so it is either $w^{t}_m$ or $v^{t,\ell,d}_{n_\ell}$. Then, a direct calculation gives
	\begin{align*} 
	&\frac{\partial}{\partial w} \mathcal{E}(A) = \frac{1}{K} \sum_{k=1}^K \frac{\partial}{\partial w}\mathcal{F}\bigl(E^{T,\bm{N},\bm{M},\bm{p},\bm{q}}_{\bm{V},\bm{W}} (\bm{x}^k), y^k\bigr)\\
	&= \frac{2}{K} \sum_{k=1}^K \bigl(E^{T,\bm{N},\bm{M},\bm{p},\bm{q}}_{\bm{V},\bm{W}} (\bm{x}^k)- y^k\bigr) \frac{\partial}{\partial w}E^{T,\bm{N},\bm{M},\bm{p},\bm{q}}_{\bm{V},\bm{W}} (\bm{x}^k),	  
	\end{align*}
	where we can simply plug in the explicit expressions for the derivatives of the ExSpliNet function described before.

	\section{Interpretation of the model}\label{sec:Interpretation}
	Let us fix $\bm{x} := [x_1, \ldots, x_D]$ and
	$$ \bm{f} (\bm{x}) := [f_1(\bm{x}),\ldots, f_O(\bm{x})].$$ 
	In the classical low-dimensional setting, it is known that each $f_o(\bm{x})$ can be efficiently approximated by means of a standard tensor-product B-spline structure, just like in the BSNN model \cite{Harris}. However, for high-dimensional data, where $D\gg1$, this is infeasible due to the excessive complexity of the tensor-product structure --- it is exponential in $D$. 

	In order to address this issue, we assume that the considered high-dimensional data belong to a lower-dimensional manifold. This is a core assumption by a variety of methods that aim at manifold learning and linear and nonlinear dimensionality reduction \cite{Andras,Hastie2}.
	Under this assumption, any component of the function $\bm{f}$ might be well approximated as
	$$ f_o(\bm{x}) \sim \sum_{t =1}^T f_{o,t}( y_{t,1}(\bm{x}), \ldots, y_{t,L}(\bm{x})),$$
	with $L \ll D$. Here, the vector of inner functions $\bm{y}_t: \R^D \to \R^L$, 
	$$\bm{x} \to \bm{y}_t(\bm{x}) := [y_{t,1}(\bm{x}), \ldots, y_{t,L}(\bm{x})],$$
	plays the role of \emph{feature extractor} that reduces the dimensionality and permits, as a next step, a standard tensor-product spline approximation $f_{o,t}$ as outer function.
	This is the main idea behind ExSpliNet. 
	
	\subsection{Inner functions}\label{sec:Interpretation-inner}
	The functions computed by the inner neural networks $\INN{\ell}$, $\ell=1,\ldots,L$, can be seen as feature extractors for dimensionality reduction and take the general form of an additive spline model:
	\begin{equation}\label{eq:inner_output}
	y_{t,\ell}(\bm{x}) := \Psi^{N_\ell,p_\ell}_{\bm{V}^{t,\ell}}(\bm{x}) = \sum_{d=1}^D \Psi^{N_\ell,p_\ell}_{\bm{v}^{t,\ell,d}}(x_d).	  
	\end{equation}
	This general form prepares the data to be further processed by tensor-product spline functions. In the following, we discuss some very particular instances of the inner neural network model.
	
	If $D$ is already small, then we can take $L = D$ and set
	$$ \Psi^{N_\ell,p_\ell}_{\bm{V}^{t,\ell}}(\bm{x}) = x_\ell, \quad \ell=1,\ldots,D,$$
	for any choice of $N_\ell>p_\ell\geq1$; see Proposition~\ref{pro:decision-vec-id} and Remark~\ref{rmk:decision-vec-id} for details.
	In such case, the input data are directly propagated to the outer neural network without modification.
	More generally, even if $D$ is not small, we can set 
	$$ \Psi^{N_\ell,p_\ell}_{\bm{V}^{t,\ell}}(\bm{x}) = x_{\sigma_t(\ell)},$$
	where
	$$\sigma_t: \{1, \ldots, L\} \to \{1, \ldots, D\}$$
	selects $L$ indices from $\{1, \ldots, D\}$ with $L\leq D$.
	In other words, each $f_{o,t}$ is fed with a subset of the input variables.
	
	Increasing the complexity, another interesting special case is a convex combination of the input variables. This can again be achieved for any choice of $N_\ell>p_\ell\geq1$. In particular, in the vein of Remark~\ref{rmk:decision-vec-id}, by choosing $p_\ell=1$ and $N_\ell=2$, we can set for some $\nu^{t,\ell,d}\geq0$ such that $\nu^{t,\ell,1}+\dots+\nu^{t,\ell,D}=1$,
	$$\Psi^{N_\ell,p_\ell}_{\bm{V}^{t,\ell}}(\bm{x}) = \sum_{d=1}^D\nu^{t,\ell,d} B_{2,1,2}(x_d) = \sum_{d=1}^D \nu^{t,\ell,d} x_d.$$
	
	We can also impose other kinds of constraints on the structure of $\Psi^{N_\ell,p_\ell}_{\bm{V}^{t,\ell}}(\bm{x})$ based on the knowledge of the particular problem to be approximated. 
	For example, we can construct a convolutional operator, similar to the classical convolutional layer, by imposing additional sparsity on the matrices $\bm{V}^{t,\ell}$, for addressing image related tasks. 
	
	Finally, we note that a large class of neural networks (including feed-forward ReLU networks and convolutional neural networks) can be interpreted as a more general additive linear spline model; see \cite{Balestriero2}. The latter spline model, however, might be highly unstructured which complicates the analysis and negatively affects the final performance of the model as observed in \cite{Balestriero2}. 
	 
	\subsection{Outer functions}\label{sec:Interpretation-outer}
	Let $\bm{y}_t:=[y_{t,1},\ldots,y_{t,L}]$, $t=1,\ldots,T$, be the low-dimensional features found using the layer of inner neural networks $\INN{\ell}$, $\ell=1,\ldots,L$; see \eqref{eq:inner_output}.	
	They are passed to the outer neural network $\ONN$ that computes the functions
	$$ f_{o,t}(\bm{y}_t) : = \Phi^{\bm{M},\bm{q}}_{\bm{w}^{o,t}}(\bm{y}_t) = \bm{w}^{o,t}\cdot \mathcal{B}_{\bm{M},\bm{q}}(\bm{y}_t),$$
	for $o=1,\ldots,O$ and $t=1,\ldots,T$, where
	$$ \mathcal{B}_{\bm{M},\bm{q}}\big( \bm{y}_t \big) = \bigotimes_{\ell=1}^L \mathcal{B}_{M_\ell,q_\ell} \big(  y_{t,\ell}(\bm{x}) \big)$$
	has components of the form
	$$ 	\prod_{\ell=1}^L B_{M_\ell,q_\ell,m_\ell}\bigl( y_{t,\ell}(\bm{x})\bigr), $$ 
	for $m_\ell=1,\ldots,M_\ell$ and $\ell=1,\ldots,L$.
	
	The multivariate B-spline vector $\mathcal{B}_{\bm{M},\bm{q}}\big(\bm{y}_t\big)$ can be interpreted as a fuzzy hierarchical partition of the domain that induces a tree structure with $L$ levels for every $t=1,\ldots,T$. This can be explained as follows.
	Let us fix $t$. For each $\ell=1,\ldots,L$, by construction,
	$\mathcal{B}_{M_\ell,q_\ell} \big( y_{t,\ell}(\bm{x}) \big)$
	is a vector of $M_\ell$ B-splines of degree $q_\ell$.
	Its components are nonnegative real values that sum up to one, and thus can be regarded as a distribution over a discrete set of hidden classes $\{c_{\ell,1},\ldots,c_{\ell,M_\ell} \}$ at level $\ell$,
	where for $m_\ell=1,\ldots,M_\ell$ we have
	$$ \mathds{P}( \bm{x} \in c_{\ell,m_\ell}) = B_{M_\ell,q_\ell,m_\ell}\bigl( y_{t,\ell}(\bm{x}) \bigr). $$
	The B-spline $B_{M_\ell,q_\ell,m_\ell}$ plays the role of \textit{decision} or \textit{gating function} at level $\ell$ based on the feature $y_{t,\ell}(\bm{x})$.
    Then, under the assumption that the events are mutually independent, the joint probability on the hierarchy of hidden classes at all levels is given by
	\begin{align*}
	&\mathds{P}(\bm{x} \in \mathcal{C}_{m_1,\ldots,m_L})
	:= \mathds{P}(\bm{x} \in c_{1,m_1} , \ldots, \bm{x} \in c_{L,m_L}) \\
	&\quad = \prod_{\ell=1}^L \mathds{P}( \bm{x} \in c_{\ell,m_\ell})
	= \prod_{\ell=1}^L B_{M_\ell,q_\ell,m_\ell}\bigl( y_{t,\ell}(\bm{x})\bigr),
	\end{align*}
	for $m_\ell=1,\ldots,M_\ell$ and $\ell=1,\ldots,L$.
	All together they form the multivariate B-spline vector
	$\mathcal{B}_{\bm{M},\bm{q}}\big(\bm{y}_t\big)$.
	A graphical representation of the induced tree structure can be found in Figure~\ref{fig:tree}.
	
	\begin{remark}
	The above is just a theoretical way to interpret the tensor product combined with the nonnegativity and the partition-of-unity property of B-splines. In practice, there is no reason why the empirical data should reflect the independence assumption used to factorize the joint probability on the hierarchy of hidden classes.
	However, such an assumption is rather common in many well-known models such as Naive Bayes models.
	\end{remark}
	
	\begin{figure}[h!]
		\centering
		\subfigure{\includegraphics[scale=0.2]{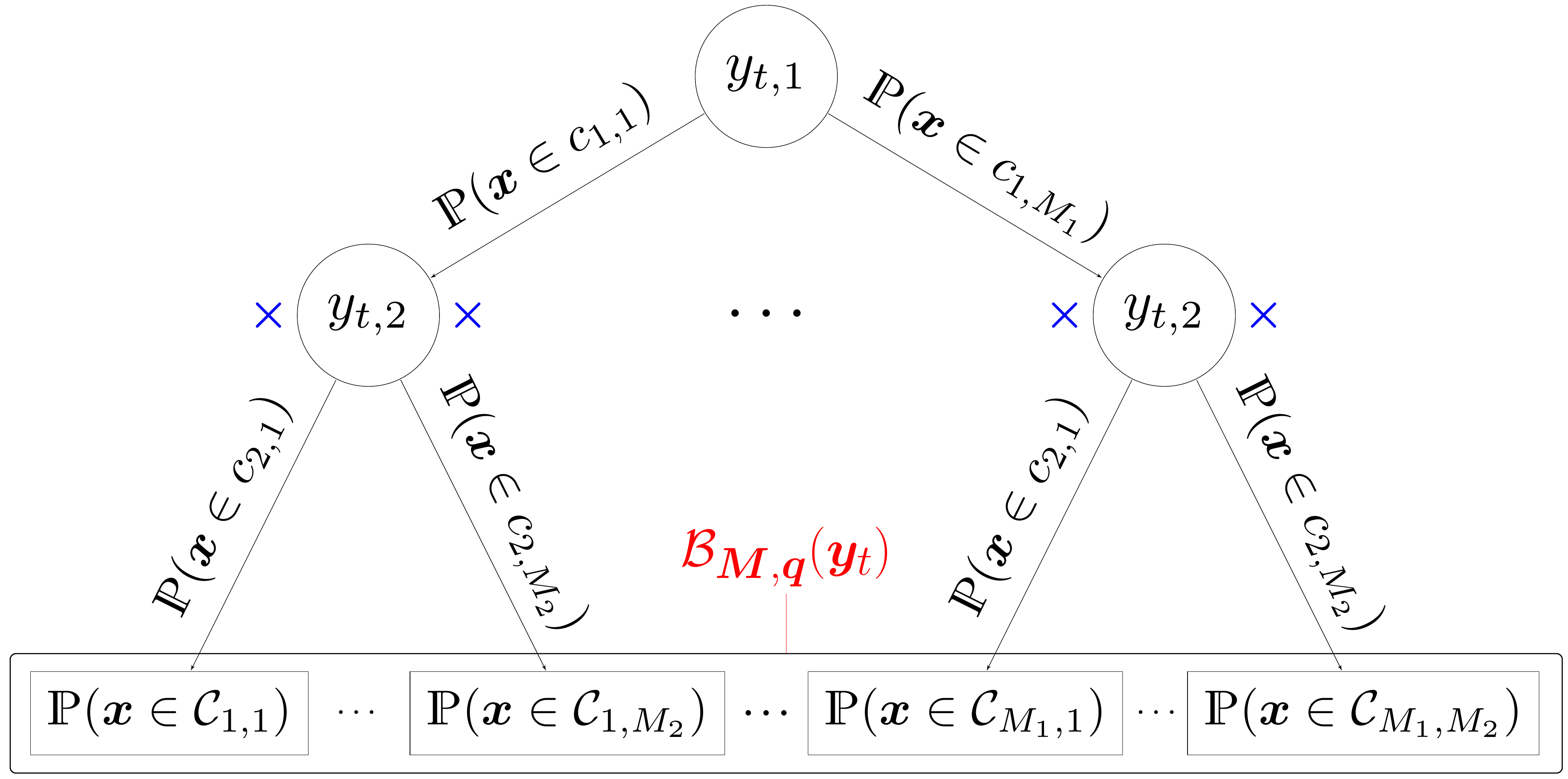}}
	\caption{A graphical representation of a tree structure (consisting of $L=2$ levels) in the outer neural network of ExSpliNet; see Figure~\ref{fig:architecture}(b).}
	\label{fig:tree}
	\end{figure}
	
	Given the probabilistic interpretation of $\mathcal{B}_{\bm{M},\bm{q}}(\bm{y}_t)$, the function $f_{o,t} (\bm{y}_t)$ is simply a weighted sum of those probabilities, i.e.,
	\begin{align*}
	&f_{o,t}(\bm{y}_t) = \bm{w}^{o,t}\cdot \mathcal{B}_{\bm{M},\bm{q}}(\bm{y}_t)\\
	&\quad = \sum_{m_1=1}^{M_1} \cdots \sum_{m_L=1}^{M_L} w^{o,t}_{m_1,\ldots,m_L} \mathds{P}(\bm{x} \in \mathcal{C}_{m_1,\ldots,m_L}).
	\end{align*}
	Furthermore, we could impose a stochasticity condition on the vectors $\bm{w}^{o,t}$ such that $w^{o,t}_{m_1,\ldots,m_L}$ represents the probability of the output class $o$ for the hierarchical hidden class $\mathcal{C}_{m_1,\ldots,m_L}$, and we end up with a final probability distribution and a mixture model for general classification problems. 
	
	The final output of ExSpliNet is computed as a superposition of the functions $f_{o,t}(\bm{y}_t)$. 
	In other words, ExSpliNet can be regarded as an ensemble of probabilistic trees induced by this fuzzy hierarchical partitions of $[0,1]^D$, based on the features $y_{t,\ell}(\bm{x})$.

	As a further step, it could be interesting to find a semantics for this partition. It could enforce interaction between rule based ontology representations like knowledge graphs and machine learning approaches similar to the mutual synergistic interaction proposed in \cite{Bellomarini}. 
	
	We conclude this section with some final observations. The tree structure described in this section is a probabilistic generalization of the classical regression tree model.
	Indeed, if we take $\bm{q} = [0,\ldots,0]$, then by the definition of $\mathcal{B}_{\bm{M},\bm{q}}$ we have
	\begin{align*}
	&f_{o,t}(\bm{y}_t) %= \bm{w}^{o,t}\cdot \mathcal{B}_{\bm{M},\bm{q}}(\bm{y}_t)\\
	 = \sum_{m_1=1}^{M_1} \cdots \sum_{m_L=1}^{M_L} w^{o,t}_{m_1,\ldots,m_L} \mathds{1}_{m_1,\ldots,m_L} (\bm{y}_t),
	\end{align*}
	where $\mathds{1}_{m_1,\ldots,m_L}$ is the indicator function on the $L$-dimensional hypercube where the constant B-spline indexed by $m_1,\ldots,m_L$ is nonzero.
	Furthermore, if $M_1 = \cdots = M_L = 2$, then the tree is binary. Finally, when taking
	$y_{t,\ell}(\bm{x}) = x_{\sigma_t(\ell)}$, we obtain an \emph{orthogonal regression tree}, while taking $y_{t,\ell}(\bm{x})$ as a convex combination of the input variables $\bm{x}$ results in an \emph{oblique regression tree}.
	
	As mentioned in the introduction, the ExSpliNet model is a feasible generalization of the BSNN model \cite{Harris} towards high-dimensional data.
	Therefore, the interpretation carried out in this section can be seen as a generalization of the fuzzy sets presented in \cite{Harris} and in \cite{WangL} where the BSNN model was applied to extract fuzzy rules for centrifugal pump monitoring.	
	One of the main differences is the addition of the inner networks. This allows for the extraction of $L$-dimensional features $\bm{y}_t$ and the application of the model even with high input dimension $D$. 
	Moreover, since we interpreted the BSNN fuzzy partition as a tree structure and ExSpliNet allows for an ensemble of these trees, we can see it as a \textit{fuzzy forest of generalized BSNNs}.

	\section{Example: application to the Iris dataset}\label{sec:Interpretation-iris}
	As illustration we apply the ExSpliNet model to the classical Iris dataset \cite{Fisher}, one of the best known datasets in the pattern recognition literature. The dataset consists of $50$ samples from each of three species of Iris flowers (Iris setosa, Iris versicolor, and Iris virginica). Four features were measured from each sample: the length and the width of the sepals and the petals. By combining these four features, the task is to distinguish the species from each other.
	
	Let us denote the $D = 4$ input features as follows: 
	\begin{itemize}
		\item $x_1$ the sepal length, 
		\item $x_2$ the sepal width,
		\item $x_3$ the petal length,
		\item $x_4$ the petal width.
	\end{itemize}
	We normalize the data so that each value belongs to the interval $[0,1]$.
	Our objective is to train a model that, given an unseen $\bm{x} = [x_1, x_2, x_3, x_4]$, leads to an output in $\R^O$ with $O=3$, where each output component is the probability of belonging to one of the three Iris species. Let us assume that
	\begin{itemize}
		\item $o=1$ stands for setosa,
		\item $o=2$ stands for versicolor,
		\item $o=3$ stands for virginica. 
	\end{itemize} 
	To accomplish this classification task we use a simple configuration of the ExSpliNet model specified by the parameters $T = 1$, $L = 2$, $p_\ell = q_\ell = 1$, $N_\ell = 2$ for $ \ell = 1,2$, and $M_1 = 2$, $M_2 =3$. In other words, for each $o = 1,2,3$, we consider the ExSpliNet function $E^{1,[2,2],[2,3],[1,1],[1,1]}_{\bm{V},\bm{W}^{o}}$ given by
	$$\Phi^{{[2,3]},{[1,1]}}_{\bm{w}^{o,1}}\Biggl( \biggl[\sum_{d=1}^4 \Psi^{2,1}_{\bm{v}^{1,\ell,d}}(x_d) \biggr]_{\ell=1,2} \Biggr).$$
	 %where we omitted the subscript $t$ since there is just one tree ($T = 1$).
	 We train this very simple configuration on 120 samples to learn the 
	 $$D T \sum_{\ell=1}^{L} N_\ell + O T \prod_{\ell=1}^L M_\ell = 34$$ 
	 weight parameters, namely
	 \begin{itemize}
	 	\item $ \bm{v}^{1,\ell,d} \in \R^2$  with $\ell=1,2$ and $d=1,2,3,4$,
	 	\item $ \bm{w}^{o,1} \in \R^{2\cdot3} = \R^6$  with $o=1,2,3$.
	 \end{itemize}
 	A test accuracy of $96.7\%$ is obtained on the remaining 30 samples.
 	
 	Now that we have a working model, we can interpret it as discussed in Section~\ref{sec:Interpretation}.
 	First of all, let us look at the new features, as in \eqref{eq:inner_output}, extracted by the inner functions:
 	$$y_\ell({\bm x}) = \sum_{d=1}^4 \Psi^{2,1}_{\bm{v}^{1,\ell,d}}(x_d), \quad \ell = 1,2, $$
 	where
	\begin{align*}
	\Psi^{2,1}_{\bm{v}^{1,\ell,d}}(x_d) 
	&= v^{1,\ell,d}_1 B_{2,1,1}(x_d) + v^{1,\ell,d}_2 B_{2,1,2}(x_d) \\
	&= v^{1,\ell,d}_1(1-x_d) + v^{1,\ell,d}_2 x_d.
	\end{align*}
	After inspecting the components of the vectors ${\bm{v}^{1,\ell,d}}$, we observe that $v_2^{1,1,3} = 0.491$,
	$v_2^{1,1,4} = 0.488$, and $v_2^{1,2,4} = 0.998$,
	while all other components are smaller than $10^{-2}$. Thus, we can approximately write the general expressions of the new features $y_1, y_2$ as
	$$y_1({\bm x}) \backsimeq (x_3 + x_4)/2 \text{  and  }  y_2({\bm x})  \backsimeq x_4.$$ 
	Figure~\ref{fig:irisFeatures} depicts the training points in terms of the new features extracted.
	
	\begin{figure}[h!]
		\centering
		\subfigure{\includegraphics[trim={1cm 0.7cm 1cm 1cm}, clip, height=5.0cm]{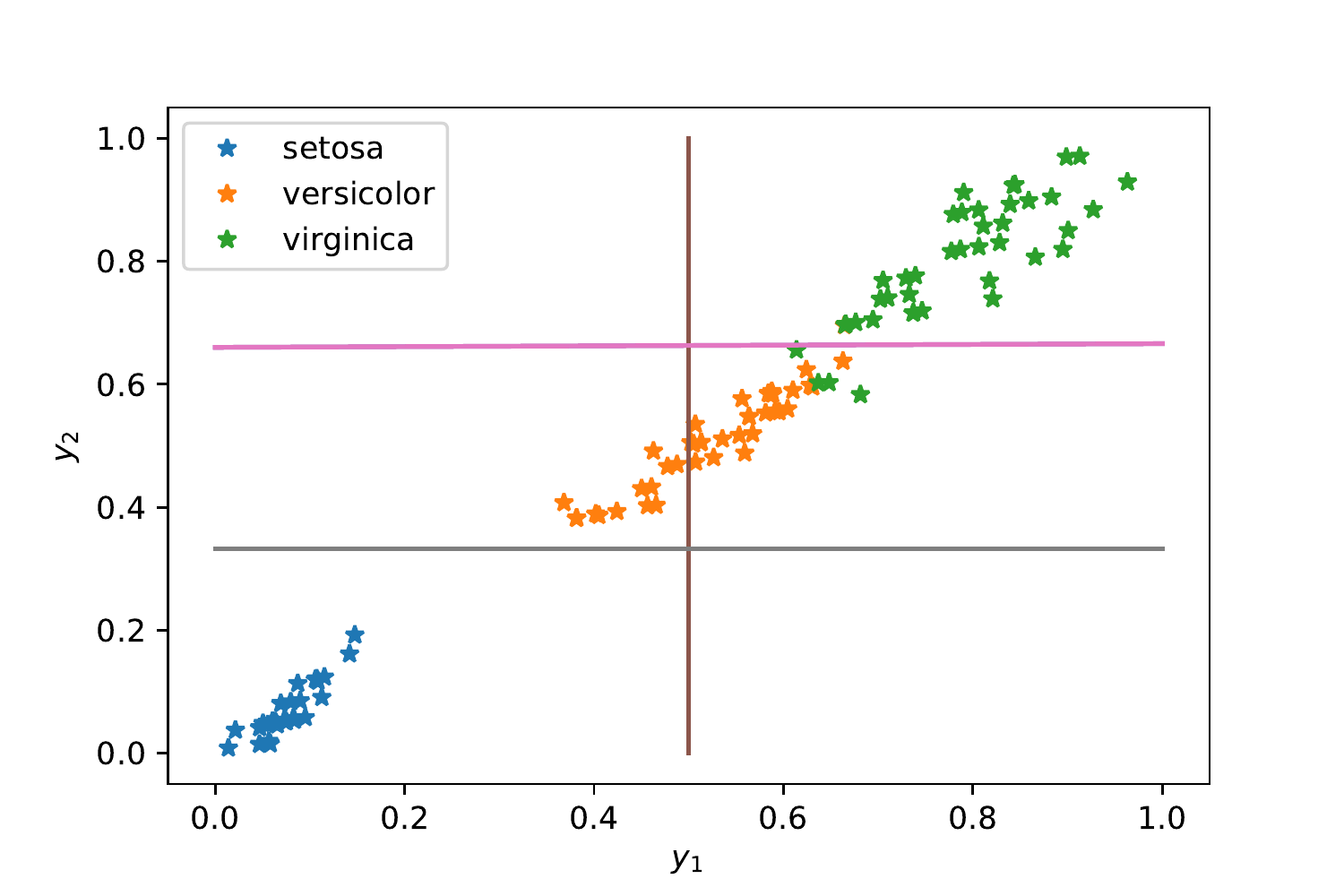}}
		\caption{Visualization of the Iris training points in terms of the new features extracted. The horizontal axis corresponds to $y_1({\bm x}) \backsimeq (x_3+x_4)/2$ and the vertical axis to $y_2({\bm x}) \backsimeq x_4$.}
		\label{fig:irisFeatures}
	\end{figure}
	
	Then, these new features are passed to the tree structure as shown in Figure~\ref{fig:tree}. In this case, the tree has $M_1 = 2$ children at the first level and $M_2 = 3$ children at the second level per tree node, resulting in a total of $M_1 \cdot M_2 = 6$ hierarchical hidden classes.
	
	The probability of belonging to the classes of the first level depends on the value of $y_1({\bm x})$:
	\begin{align*}
	\mathds{P}(\bm{x} \in c_{1,1}) &= B_{2,1,1}(y_1({\bm x})) = 1-y_1({\bm x}), \\ 
	\mathds{P}(\bm{x} \in c_{1,2}) &= B_{2,1,2}(y_1({\bm x})) = y_1({\bm x}).
	\end{align*}
	In other words, a point $\bm{x}$ is more likely to belong to the first class of the first level $c_{1,1}$, the more the new feature $y_1(\bm{x})$ lies at the left of the vertical line in Figure~\ref{fig:irisFeatures}. Otherwise, $\bm{x}$ is more likely to belong to the second class $c_{1,2}$.
	
	Similarly, one can compute
	\begin{align*}
	\mathds{P}(\bm{x} \in c_{2,1}) &= B_{3,1,1}(y_2({\bm x})), \\ 
	\mathds{P}(\bm{x} \in c_{2,2}) &= B_{3,1,2}(y_2({\bm x})), \\
	\mathds{P}(\bm{x} \in c_{2,3}) &= B_{3,1,3}(y_2({\bm x})),
	\end{align*}
	in order to obtain the probability of belonging to one of the three children of the second level. Thus, looking at the horizontal lines in Figure~\ref{fig:irisFeatures}, one can understand the most likely class based on the value of the second feature $y_2({\bm x})$.
	Recall that
	\begin{align*}
	 B_{3,1,1}(y) &= \begin{cases}
		1-2y, & y\in[0,1/2),\\
		0 & y\in[1/2,1],
    \end{cases}\\
	 B_{3,1,2}(y) &= \begin{cases}
		2y, & y\in[0,1/2),\\
		2-2y, & y\in[1/2,1],
    \end{cases}\\
	 B_{3,1,3}(y) &= \begin{cases}
		0 & y\in[0,1/2),\\
		2y-1, & y\in[1/2,1].
    \end{cases}    
	\end{align*}

	The trained weights $\bm{w}^{o,1} \in \R^6$ with $o=1,2,3$ give, for each hierarchical hidden class, the probability of the Iris species (related to $o$) in that class. 
	
	Finally, we compute 
	\begin{align*}
	f_{o}(y_1, y_2) %= \bm{w}^{o,1}\cdot \mathcal{B}_{\bm{M},\bm{q}}([y_1, y_2]) 
	=\sum_{m_1=1}^{2} \sum_{m_2=1}^{3} w^{o,1}_{m_1,m_2} \mathds{P}(\bm{x} \in \mathcal{C}_{m_1,m_2}).
	\end{align*}
	The prediction of a new observation is performed by choosing the most probable Iris species,
	$${\underset{o = 1,2,3}{\arg\max}}\ f_{o}.$$
	
	This can be rephrased as a set of probabilistic rules of the form
	\begin{itemize}
		\item $[w^{o,1}_{m_1,m_2}] : (\bm{x} \in c_{1,m_1}) \wedge (\bm{x} \in  c_{2,m_2}) \Rightarrow {\bm x} \text{  is  } o$;
		\item $[\mathds{P}(\bm{x} \in c_{1,m_1})] : \bm{x} \in c_{1,m_1}$;
		\item $[\mathds{P}(\bm{x} \in c_{2,m_2})] : \bm{x} \in c_{2,m_2}$.
	\end{itemize}
	Here, the rule is formulated after the colon and its probability is given in square brackets before the colon.
	Examples are
	\begin{itemize}
		\item $ [1.0] : (\bm{x} \in c_{1,1}) \wedge (\bm{x} \in  c_{2,1}) \Rightarrow \bm{x} \text{  is setosa}$, 
		or, approximately, if $(x_3 + x_4)/2 < 1/2$ and $x_4 < 1/3$ then $\bm{x}$ is setosa.
		\item  $[0.9] : (\bm{x} \in c_{1,2}) \wedge (\bm{x} \in  c_{2,2}) \Rightarrow \bm{x} \text{  is  versicolor}$,
		or, approximately, 
		if $(x_3 + x_4)/2 > 1/2$ and $1/3 < x_4 < 2/3 $ then $\bm{x}$ is versicolor.
	\end{itemize}
	Such a probabilistic ontology can then be used in practice by modern probabilistic ontology solvers to query not only the final decision but also the reasoning behind it by means of the path of rules used to infer the final decision. For example,
	\begin{itemize}
        \item $\bm{x} = [x_1, x_2, x_3, x_4] $ is a setosa, because
			its petal length $x_3$ and width $x_4$ are small.
	\end{itemize}
	
	The Iris dataset is a simple example that allows us to analyze it by manually looking at the weight parameters. The described method, however, can be automated to tackle more complicated tasks.
	In such cases a single tree may not be sufficient but pruning techniques may be applied to reduce their number while controlling the trade-off between simplicity (interpretability) and accuracy (expressivity) of the model.

	\section{Expressive power results}\label{sec:Approximation}
	In this section we focus on the expressive power of the proposed network model.
	We show two universal approximation results, the first based on the KST and the second on multivariate spline theory. 
    Let us again consider $O = 1$ for simplicity of presentation but without loss of generality, so that we can drop all superscripts $o$ in our notations. 
	Let $\|\cdot\|_k$ denote the standard $L_k$-norm on the unit domain $[0,1]^D$ for some $1\leq k \leq \infty$ and some $D\geq1$.
	
	\subsection{Kolmogorov approximation}\label{sec:Approx-Kolmogorov}
	Kolmogorov's representation \eqref{eq:KST} can be simulated by means of an ExSpliNet function with univariate splines as inner and outer functions.
	Indeed, for $T = 2D+1$ and $L = 1$, we have
	\begin{equation*} %\label{eq:kolmogorov-spline}
	E^{2D+1,\bm{N},\bm{M},\bm{p},\bm{q}}_{\bm{V},\bm{W}} (\bm{x}) = 
	\sum_{t=1}^{2D+1} \Phi^{\bm{M},\bm{q}}_{\bm{w}^{t}}\biggl( \sum_{d=1}^D \Psi^{N_1,p_1}_{\bm{v}^{t,1,d}}(x_d)\biggr).
	\end{equation*}
	Note that $\bm{N}=[N_1]$, $\bm{M}=[M_1]$, $\bm{p}=[p_1]$, $\bm{q}=[q_1]$, and $\Phi^{\bm{M},\bm{q}}_{\bm{w}^{t}}$ is just a univariate spline. 
	Due to the fact that polynomials, and a fortiori splines, are dense in the space of continuous functions, we arrive at the following approximation theorem.
	
	\begin{theorem}\label{thm:kolmogorov-approx}
	Let $f: [0,1]^D\to \R$ be a continuous function. For any $\epsilon >0$, there exists an ExSpliNet function such that  
	$$\bigl\| f-E^{2D+1,\bm{N},\bm{M},\bm{p},\bm{q}}_{\bm{V},\bm{W}} \bigr\|_k \leq \epsilon,$$
	for any $1\leq k\leq\infty$.
	\end{theorem}
	\begin{proof}
	Let us fix an arbitrary $\epsilon>0$. The KST states that there exists a decomposition of the form
	$$ f(\bm{x}) = \sum_{t = 1}^{2D + 1} \Phi_t\biggl( \sum_{d=1}^{D} \Psi_{t,d}(x_d) \biggr).$$
	In the following, we will construct an ExSpliNet function that approximates this decomposition with accuracy $\epsilon$.
	
	Given any $\delta_t>0$, by the continuity of $\Psi_{t,d}$ on $[0,1]$ and the density of splines in $C([0,1])$ there exists a spline $\widetilde\Psi^{N_1,p_1}_{\bm{v}^{t,1,d}}$ with $p_1\geq1$  (without loss of generality we can assume the same $N_1, p_1$ for each $t, d$) such that
	$$ \| \Psi_{t,d} - \widetilde\Psi^{N_1,p_1}_{\bm{v}^{t,1,d}} \|_\infty \leq \frac{\delta_t}{D}$$
	and
	$$\operatorname{range}\bigl(\widetilde\Psi^{N_1,p_1}_{\bm{v}^{t,1,d}}\bigr)\subseteq\operatorname{range}\bigl(\Psi_{t,d}\bigr)\subseteq[0,1].$$
	Specifically, this can be accomplished by means of the Schoenberg spline operator \cite{Lyche,Marsden}, i.e., a spline where the weights take the form
	$$v^{t,1,d}_{n_1}=\Psi_{t,d}(\xi^*_{N_1,p_1,n_1}).$$
	Recall that the $\xi^*_{N_1,p_1,n_1}$ are the Greville abscissae defined in \eqref{eq:greville}.
	We define the scaling function
	$$\chi_t(x):=x \Big / \max\biggl(1,\sum_{d=1}^{D} \sum_{n_1=1}^{N_1} v^{t,1,d}_{n_1}\biggr),$$
	and we set
	$$\Psi^{N_1,p_1}_{\bm{v}^{t,1,d}}(x_d):=\chi_t\bigl(\widetilde\Psi^{N_1,p_1}_{\bm{v}^{t,1,d}}(x_d)\bigr).$$
	Note that the weights of the above spline are given by $\chi_t(v^{t,1,d}_{n_1})$ and thus satisfy the (sufficient) conditions mentioned in Remark~\ref{rmk:V-condition-simple} for ensuring a valid set of inner spline functions.
	
	Furthermore, by the continuity of $\Phi_t$ on $[0,1]$ and the density of splines in $C([0,1])$ there exists a spline $\Phi^{\bm{M},\bm{q}}_{ \bm{w}^{t}}$ with $\bm{q}\geq1$ (without loss of generality we can assume the same $\bm{M},\bm{q}$ for each $t$) such that
	$$ \bigl\| \Phi_t - \widetilde\Phi^{\bm{M},\bm{q}}_{\bm{w}^{t}} \bigr\|_\infty \leq \frac{\epsilon}{2(2D+1)},$$
	where 
	$$\widetilde\Phi^{\bm{M},\bm{q}}_{\bm{w}^{t}}(x):=\Phi^{\bm{M},\bm{q}}_{\bm{w}^{t}}(\chi_t(x)).$$
	By the uniform continuity of $\Phi^{\bm{M},\bm{q}}_{\bm{w}^{t}}$, and thus $\widetilde\Phi^{\bm{M},\bm{q}}_{\bm{w}^{t}}$, on $[0,1]$ there exists $\delta_t>0$ independent of the chosen $x,y$ such that $|x-y| \leq \delta_t$ implies
	$$ \bigl| \widetilde\Phi^{\bm{M},\bm{q}}_{\bm{w}^{t}}(x) - \widetilde\Phi^{\bm{M},\bm{q}}_{ \bm{w}^{t}}(y) \bigr| \leq \frac{\epsilon}{2(2D+1)}.$$
	Note that
	$$\Phi^{\bm{M},\bm{q}}_{\bm{w}^{t}}\biggl( \sum_{d=1}^D \Psi^{N_1,p_1}_{\bm{v}^{t,1,d}}\biggr)=\widetilde\Phi^{\bm{M},\bm{q}}_{\bm{w}^{t}}\biggl( \sum_{d=1}^D \widetilde\Psi^{N_1,p_1}_{\bm{v}^{t,1,d}}\biggr).$$
	
	Finally, we have
	\begin{align*}
	&\bigl\| f-E^{2D+1,\bm{N},\bm{M},\bm{p},\bm{q}}_{\bm{V},\bm{W}} \bigr\|_k
	  \leq \bigl\| f-E^{2D+1,\bm{N},\bm{M},\bm{p},\bm{q}}_{\bm{V},\bm{W}} \bigr\|_\infty\\
	& \leq \sum_{t=1}^{2D+1} \Biggl\| \Phi_t\biggl( \sum_{d=1}^{D} \Psi_{t,d} \biggr)  - \Phi^{\bm{M},\bm{q}}_{\bm{w}^{t}}\biggl( \sum_{d=1}^D \Psi^{N_1,p_1}_{\bm{v}^{t,1,d}}\biggr) \Biggr\|_\infty.
	\end{align*}
	This upper bound is less than or equal to
	\begin{align*}
	&\sum_{t=1}^{2D + 1} \Biggl\| \Phi_t\biggl( \sum_{d=1}^{D} \Psi_{t,d}\biggr)  - \widetilde\Phi^{\bm{M},\bm{q}}_{\bm{w}^{t}}\biggl(\sum_{d=1}^{D} \Psi_{t,d} \biggr) \Biggr\|_\infty\\
	&+ \sum_{t=1}^{2D+1} \Biggl\| \widetilde\Phi^{\bm{M},\bm{q}}_{\bm{w}^{t}}\biggl( \sum_{d=1}^{D} \Psi_{t,d} \biggr)  - \widetilde\Phi^{\bm{M},\bm{q}}_{\bm{w}^{t}}\biggl( \sum_{d=1}^D \widetilde\Psi^{N_1,p_1}_{\bm{v}^{t,1,d}}\biggr) \Biggr\|_\infty,
	\end{align*}
	which, in turn, is less than or equal to $\epsilon$ by our choice of inner and outer functions. This concludes the proof.
	\end{proof}
	
	From Proposition~\ref{pro:params} we deduce that the number of weight parameters needed to express $E^{2D+1,\bm{N},\bm{M},\bm{p},\bm{q}}_{\bm{V},\bm{W}}$ in Theorem~\ref{thm:kolmogorov-approx} is equal to
	$$ D (2D + 1) N_1 + (2D + 1) M_1.$$
	The proof of the theorem does not provide us explicit values of 
	$N_1,M_1$ that achieve the required accuracy $\epsilon$. Also the degrees $p_1,q_1$ are not specified.
	In the next section we consider another subclass of the model ExSpliNet that allows for approximation results with  explicit estimation of the parameters.
	
	\subsection{Multivariate spline approximation}
	In this section we demonstrate the expressivity of ExSpliNet by exploiting known approximation results for multivariate splines \cite{Schumaker}. 
	Let $\mathcal{L}^r_k$ be the Sobolev function space equipped with the norm that is a combination of $L_k$-norms of the function together with its partial derivatives up to order $r\geq1$, i.e.,
	\begin{align*}
	\mathcal{L}^r_k:=\biggl\{f:[0,1]^D\rightarrow \R: \biggl\|\frac{\partial^{r_1}}{\partial x_1^{r_1}}\cdots\frac{\partial^{r_D}}{\partial x_D^{r_D}}f\biggr\|_k<\infty,\\
	\forall\, 0\leq r_1+\cdots+r_D\leq r\biggr\}.
	\end{align*}
	Fix $\bm{q} := [q_1, \ldots, q_D]\in\Z^D$ with each $q_d\geq0$, and $\bm{M} := [M_1, \ldots, M_D]\in\Z^D$ with each $M_d> q_d$. For any smooth function $f\in\mathcal{L}^r_k : [0,1]^D \rightarrow \R$, there exists a vector $\bm{w}\in\R^{M_1\cdots M_D}$ such that 
	\begin{equation}\label{eq:spline_approx}
	\|f-s^{\bm{M},\bm{q}}_{\bm{w}}\|_k \leq C_k \sum_{d=1}^D (h_d)^r \Bigl\|\frac{\partial^r}{\partial x_d^r}f\Bigr\|_k,
	\end{equation}
	for any $q_d \geq r-1$, where
	\begin{equation}\label{eq:h}
	h_d:=\frac{1}{M_d-q_d}, \quad d=1,\ldots,D,
	\end{equation}
	and $C_k$ is a constant independent of $f$ and $h_d$, but may depend on $r$ and $q_d$. In the recent work \cite{Sande1}, the following explicit constant has been derived in case of the $L_2$-norm: 
	$$C_2=\Bigl(\frac{1}{\pi}\Bigr)^r.$$

	Any multivariate spline of the form $s^{\bm{M},\bm{q}}_{\bm{w}}$ can be represented as an ExSpliNet function. Indeed, for $T = 1$ and $L = D$, we have
	$$E^{1,\bm{N},\bm{M},\bm{p},\bm{q}}_{\bm{V},\bm{W}} (\bm{x}) =
	\bm{w}^1 \cdot \mathcal{B}_{\bm{M},\bm{q}}\bigl( \Psi^{\bm{N},\bm{p}}_{\bm{V}^1} (\bm{x}) \bigr). $$
	Hence, we obtain
	$$E^{1,\bm{N},\bm{M},\bm{p},\bm{q}}_{\bm{V},\bm{W}} (\bm{x}) = s^{\bm{M},\bm{q}}_{\bm{w}} (\bm{x})$$
	by taking $\bm{w}^1=\bm{w}$ and choosing $\bm{V},\bm{N},\bm{p}$ according to the following proposition.
	
	\begin{proposition}\label{pro:decision-vec-id}
	For any given $\bm{p} \in \Z^L$ and $\bm{N} \in \Z^L$ with each $N_\ell>p_\ell\geq1$ and $L=D$, there exist vectors $\bm{v}^{t,\ell,d} \in \R^{N_\ell}$ such that
	\begin{equation}\label{eq:decision-vec-id}
	\Psi^{\bm{N},\bm{p}}_{\bm{V}^t} (\bm{x}) = \bm{x}, \quad \bm{x} \in [0,1]^D.
	\end{equation}
	These vectors satisfy the assumption \eqref{eq:V-condition}.
	\end{proposition}
	\begin{proof}
	Fix $\bm{x} \in [0,1]^D$. By construction we have
	$$\Psi^{N_\ell,p_\ell}_{\bm{V}^{t,\ell}}(\bm{x})
	= \sum_{d=1}^D  \bm{v}^{t,\ell,d} \cdot \mathcal{B}_{N_\ell,p_\ell}(x_d).$$
	From \eqref{eq:identity-Bspline} we deduce that the choice
	\begin{equation}\label{eq:choice-v}
	\begin{aligned}
	\bm{v}^{t,\ell,\ell} &=[\xi^*_{N_\ell,p_\ell,1},\ldots,\xi^*_{N_\ell,p_\ell,N_\ell}], \\ 
	\bm{v}^{t,\ell,d} &=[0,\ldots,0], \quad d \neq l,
	\end{aligned}
	\end{equation}
	results in
	$$\Psi^{N_\ell,p_\ell}_{\bm{V}^{t,\ell}}(\bm{x}) = x_\ell,
	$$
	and, as a consequence, we get \eqref{eq:decision-vec-id}. The choice in \eqref{eq:choice-v} satisfies the conditions in \eqref{eq:V-condition}. This can be verified by taking $\nu^{t,\ell,\ell}=\xi^*_{N_\ell,p_\ell,N_\ell}=1$ and $\nu^{t,\ell,d}=0$ for $d\neq \ell$.
	\end{proof}
	\begin{remark}\label{rmk:decision-vec-id}
	A particularly simple instance of \eqref{eq:decision-vec-id} is obtained by taking $p_\ell=1$ and $N_\ell=2$. In this case,
	$$\Psi^{N_\ell,p_\ell}_{\bm{V}^{t,\ell}}(\bm{x}) = B_{2,1,2}(x_\ell) = x_\ell.$$
	\end{remark}
	
	We then arrive at the following approximation result.
	\begin{theorem}\label{thm:spline-approx}
	For any smooth function $f\in\mathcal{L}^r_k : [0,1]^D \rightarrow \R$, there exists an ExSpliNet function such that 
	$$\|f-E^{1,\bm{N},\bm{M},\bm{p},\bm{q}}_{\bm{V},\bm{W}}\|_k \leq C_k \sum_{d=1}^D (h_d)^r\Bigl\|\frac{\partial^r}{\partial x_d^r}f\Bigr\|_k,$$
	where $h_d$ is defined in \eqref{eq:h} and $C_k$ is the constant specified in \eqref{eq:spline_approx}. In particular, for $k=2$ it holds $C_2=(\frac{1}{\pi})^r$.
	\end{theorem}
	
	From Proposition~\ref{pro:params} we deduce that the number of weight parameters needed to express $E^{1,\bm{N},\bm{M},\bm{p},\bm{q}}_{\bm{V},\bm{W}}$ in Theorem~\ref{thm:spline-approx} is equal to
	$$ D\sum_{d=1}^{D} N_d + \prod_{d=1}^D M_d.$$
	Theorem~\ref{thm:spline-approx} implies that for any $f\in\mathcal{L}^r_k$ and $\epsilon >0$, one can construct an ExSpliNet function such that
	$$\|f-E^{1,\bm{N},\bm{M},\bm{p},\bm{q}}_{\bm{V},\bm{W}}\|_k \leq \epsilon,$$
	if 
	$$M_d-q_d\geq \biggl( \frac{C_k}{\epsilon}{ \sum_{d=1}^D\Bigl\|\frac{\partial^r}{\partial x_d^r}f\Bigr\|_k} \biggr)^{1/r}, \quad d=1,\ldots,D,$$
	for any choice of $\bm{N},\bm{p},\bm{q}$ such that each $N_d>p_d \geq 1$ and $q_d \geq r-1$.
	
	\begin{remark}\label{rmk:L-intermediate}
	In this section we considered $L=D$. However, working with $D$-variate splines might be in practice computationally too expensive for high values of $D$. On the other hand, in Section~\ref{sec:Approx-Kolmogorov} we considered $L=1$. In both cases we derived universal approximation results, so we believe that similar results could also be obtained for intermediate cases $1<L<D$.
	\end{remark}

	\section{Implementation and experiments}\label{sec:Experiments}
	The ExSpliNet model can be easily implemented in the Python ecosystem by combining standard B-spline code (for example, the class {scipy.interpolate.BSpline}) with a standard deep learning environment (for example, Keras/Tensorflow or PyTorch).
	In our implementation we considered a simplified version of the conditions in \eqref{eq:V-condition}, see Remark~\ref{rmk:V-condition-simple}, by imposing
	\begin{equation*}%\label{eq:V-condition-simplest}
	0\leq v^{t,\ell,d}_{n_\ell}, \quad \sum_{d=1}^{D} \sum_{n_\ell=1}^{N_\ell} v^{t,\ell,d}_{n_\ell} = 1.
	\end{equation*}
	Such constrained problem can be transformed into an unconstrained problem by taking
	$$ v^{t,\ell,d}_{n_\ell} = \frac{ (u^{t,\ell,d}_{n_\ell})^2 }{\sum_{d=1}^{D} \sum_{n_\ell=1}^{N_\ell} \bigl(u^{t,\ell,d}_{n_\ell}\bigr)^2 }, $$
	for any given $u^{t,\ell,d}_{n_\ell}$. This allows for the use of standard unconstrained optimization algorithms via (stochastic) gradient descent, such as the Adam optimizer \cite{Kingma}.
	Alternatively, this kind of conditions can be enforced by using the \emph{layer weight constraints} functionalities built in both mentioned deep learning environments.
	
	In the following, we illustrate the performance of the model, based on this implementation, on a small selection of synthetic approximation tasks and classical machine learning tasks. 
	
	\subsection{Approximation tasks}
	In our first set of experiments, we consider data-driven function approximation and the approximate solution of differential problems in the PINN framework \cite{Raissi}. We compare the ExSpliNet model for those two tasks with classical feed-forward neural networks (FFNNs). We see that the enhanced approximation properties of B-splines generally lead to a reduced complexity of the network and/or better accuracy.
	
	\paragraph{Function approximation} 
	The function approximation problem is addressed as a risk minimization problem of the form \eqref{eq:risk} using the squared loss function and the training dataset is chosen to be a random and uniform sampling of the function.

	Three types of network configurations are compared to approximate two functions $u$ on a domain $\Omega\subset\R^D$ for $D=1$ and $D=4$, respectively. Note that $O=1$ in this case. In a preprocessing step, the domain $\Omega$ may need to be rescaled so that it belongs to the unit domain $[0,1]^D$.
\begin{itemize}[leftmargin=*]
	\item FFNN configurations:
	\begin{itemize}[leftmargin=*]
		\item layers $l \in [2, 4, 6, \ldots, 20]$,
		\item neurons per layer  $n \in [10, 30, 50, \ldots, 110]$,
		\item activation functions $a \in [\ReLU, \sigmoid, \tanh]$.
	\end{itemize}
	\item First type of ExSpliNet configurations with fixed degrees $p_\ell = q_\ell = 1$, for a direct comparison with FFNN based on linear ReLU functions:
	\begin{itemize}[leftmargin=*]
		\item levels $L \in [2, 3]$,
		\item trees $T \in [5, 20]$,
		\item inner B-splines per level $N_\ell = [5, 10, 30, 50]$,
		\item outer B-splines per level $M_\ell \in [5, 10]$.
	\end{itemize}
	\item Second type of ExSpliNet configurations with fixed higher degrees $p_\ell = q_\ell = 3$ and the other parameters are the same as for the first ExSpliNet configurations.
\end{itemize}
	For each of the above three choices, only the best configuration, with the smallest test mean squared error (Test MSE), is mentioned in the following experiments. Both the Test MSE and the required number of parameters are reported. In all the experiments we used the Adam optimizer with a learning rate of $0.001$ and $15$ epochs.
	We remark that we have also cross-validated the obtained best configurations, but since the results were very comparable we do not mention them in the experiments.
	
\begin{experiment}\label{exp:1}
	We randomly and uniformly sampled $5,000$ points for training and $2,500$ for testing from the function
	$$u(x) = \cos(20\pi x),  \quad  x \sim U[0,1].$$
	\begin{itemize}[leftmargin=*]
		\item Best FFNN: $l= 10$, $n=30$, $a=\ReLU$, \\ Parameters $ = 7,531$, Test MSE $ = 3.96 \cdot 10^{-1} $.
		\item Best ExSpliNet 1: $T= 20$, $L=3$, $N_\ell=30$, $M_\ell=5$, \\ Parameters $ = 4,300$, Test MSE $ = 6.99 \cdot 10^{-5}$.
		\item Best ExSpliNet 2: $T= 5$, $L=3$, $N_\ell=50$, $M_\ell = 10$, \\ Parameters $ = 5,750$, Test MSE $ =1.80 \cdot 10^{-6}$.
	\end{itemize}
\end{experiment}
\begin{experiment}\label{exp:2}
	We randomly and uniformly sampled $50,000$ points for training and $25,000$ for testing from the function
	$$u(\bm{x}) = x_1 + x_2^2 + x_3^3 +  e^{x_4}  + x_1x_2 + x_3x_4, \quad  x_d \sim U[-1,1].$$	
	\begin{itemize}[leftmargin=*]
	\item Best FFNN: $l= 4$, $n=50$, $a=\ReLU$, \\ Parameters $ = 5,401$, Test MSE $ = 2.56 \cdot 10^{-4}$.
	\item Best ExSpliNet 1: $T= 5$, $L=3$, $N_\ell=5$, $M_\ell=5$, \\ Parameters $ = 925$, Test MSE $ = 1.28 \cdot 10^{-4}$.
	\item Best ExSpliNet 2: $T=20$, $L=2$, $N_\ell=5$, $M_\ell=5$, \\ Parameters $ = 1,300$, Test MSE $ = 1.31 \cdot 10^{-5}$.
	\end{itemize}
\end{experiment}
	
	\paragraph{Physics-informed neural network}
	In the PINN framework, a given boundary-value differential problem is transformed into a risk minimization problem, where the so-called differential empirical risk is composed as the sum of two empirical risks; the first takes care of the differential problem in the interior of the domain and the other takes care of the boundary conditions.

	For example, consider the second-order differential equation
	\begin{equation}\label{eq:pde}
	 -\Delta u(\bm{x}) = f(\bm{x}), \quad  \bm{x} \in \Omega\subset\R^D, 
	\end{equation}
	with boundary condition $u(\bm{x}) = g(\bm{x})$ for $\bm{x}\in\partial\Omega$.
	The idea is to take the neural network as a model of the approximate solution $\hat{u}$. 
	Suppose we have a set of collocation points, consisting of interior points $A_i := \{\bm{x}_i^1,\ldots,\bm{x}_i^{K_i} \}$ belonging to $\Omega$ and boundary points $A_b := \{\bm{x}_b^1,\ldots,\bm{x}_b^{K_b} \}$ belonging to $\partial\Omega$. Then, the \textit{differential empirical risk} $\mathcal{E}(A_i,A_b)$ takes the following form:
	$$ \mathcal{E}(A_i,A_b) := \mathcal{E}_i(A_i) + \lambda \mathcal{E}_b(A_b), $$
	for some $\lambda>0$ and
	\begin{align*}
	\mathcal{E}_i(A_i) &:= \frac{1}{K_i} \sum_{k=1}^{K_i} \bigl(-\Delta \hat{u}(\bm{x}_i^k) - f(\bm{x}_i^k)\bigr)^2,\\
	\mathcal{E}_b(A_b) &:= \frac{1}{K_b} \sum_{k=1}^{K_b} \bigl(\hat{u}(\bm{x}_b^k)-g(\bm{x}_b^k)\bigr)^2.
	\end{align*}
	During training this function is minimized on a set of randomly chosen collocation points.

	We remark that the activation functions need to be at least $C^2$ smooth to solve the differential problem in \eqref{eq:pde}. For this reason, the ReLU activation function is not suitable in FFNN, and it is common to rely on the tanh activation function in the PINN framework. Regarding ExSpliNet, the degree of the inner and outer B-splines must be at least three to ensure $C^2$ smoothness. Again, in a preprocessing step, the domain $\Omega$ may need to be rescaled so that it belongs to the unit domain $[0,1]^D$.

	In the following experiments, we set $\lambda = 10^{4}$ and used the Adam optimizer with a learning rate of $0.001$ and $5,000$ epochs.

\begin{experiment}\label{exp:3}
	We solved the differential equation
	\begin{equation}\label{eq:pde-exp}
	 -u''(x) = 4\pi^2 \sin (2\pi x), \quad x \in (0,1),
	\end{equation}
	with boundary condition $u(0) = u(1) = 0$.
	The exact solution of this problem is 
	$ u(x) = \sin (2 \pi x)$.
	For each run we took $1,000$ collocation points. The results obtained by different configurations of FFNN with tanh activation functions are summarized in Table~\ref{tab:NN3}, while the results obtained by different configurations of ExSpliNet with $p_\ell=q_\ell=3$ and $L=2$ are summarized in Table~\ref{tab:Ex3}. The three best results are indicated in boldface.
	In these tables, ``DER'' stands for the differential empirical risk after training, while ``MSE'' stands for the mean square error between the exact solution and the trained network model on $300$ uniformly spaced points in the domain.
\end{experiment}
\begin{table}[t!]
	\centering\small
	\begin{tabular}{rrrrr} 
		\hline
		$l$ & $n$ & Params & DER & MSE \\
		\hline & & & &\\[-1.9ex]
		 $4$ & $10$ &   $251$ & $\SciNot{1.61}{-1}$ & $\SciNot{1.24}{-5}$  \\
		 $4$ & $20$ &   $901$ & $\SciNot{7.26}{-2}$ & \bm{$\SciNot{4.98}{-6}$}  \\
		 $4$ & $30$ & $1,951$ & $\SciNot{1.14}{-1}$ & \bm{$\SciNot{4.17}{-6}$}  \\
		 $6$ & $10$ &   $471$ & $\SciNot{4.46}{-1}$ & $\SciNot{1.35}{-4}$  \\
		 $6$ & $20$ & $1,741$ & $\SciNot{1.36}{-1}$ & $\SciNot{8.57}{-6}$  \\
		 $6$ & $30$ & $3,811$ & $\SciNot{2.03}{+0}$ & $\SciNot{1.22}{-4}$  \\
		 $8$ & $10$ &   $691$ & $\SciNot{7.99}{-2}$ & $\SciNot{5.52}{-5}$  \\
		 $8$ & $20$ & $2,581$ & $\SciNot{4.08}{-2}$ & \bm{$\SciNot{3.25}{-6}$}  \\
		 $8$ & $30$ & $5,671$ & $\SciNot{1.82}{+0}$ & $\SciNot{4.79}{-5}$  \\
		$10$ & $10$ &   $911$ & $\SciNot{8.64}{-1}$ & $\SciNot{1.42}{-5}$  \\
		$10$ & $20$ & $3,421$ & $\SciNot{9.50}{-2}$ & $\SciNot{4.99}{-6}$  \\
		$10$ & $30$ & $7,531$ & $\SciNot{1.35}{+0}$ & $\SciNot{1.05}{-4}$ 
	\end{tabular}\caption{Experiment~\ref{exp:3}: Results for differential problem \eqref{eq:pde-exp} obtained by different configurations of FFNN with tanh activation functions.} \label{tab:NN3}
\end{table}
\begin{table}[t!]
	\centering\small
	\begin{tabular}{rrrrrr} \hline
		$T$ & $N_\ell$ & $M_\ell$ & Params & DER & MSE  \\ 
		\hline & & & & &\\[-1.9ex]
		 $5$ &  $5$ &  $5$ &   $175$ & $\SciNot{1.00}{-3}$ & $\SciNot{5.11}{-10}$ \\
		 $5$ &  $5$ & $10$ &   $550$ & $\SciNot{2.03}{-2}$ & $\SciNot{6.62}{-9}$  \\
		 $5$ &  $5$ & $20$ & $2,050$ & $\SciNot{9.43}{-2}$ & $\SciNot{3.90}{-8}$  \\
		
		 $5$ & $10$ &  $5$ &   $225$ & $\SciNot{6.91}{-3}$ & $\SciNot{4.81}{-10}$ \\
		 $5$ & $10$ & $10$ &   $600$ & $\SciNot{6.01}{-2}$ & $\SciNot{1.55}{-7}$  \\
		 $5$ & $10$ & $20$ & $2,100$ & $\SciNot{1.72}{-1}$ & $\SciNot{6.85}{-7}$  \\
		
		$10$ &  $5$ &  $5$ &   $350$ & $\SciNot{3.34}{-4}$ & \bm{$\SciNot{1.11}{-10}$} \\
		$10$ &  $5$ & $10$ & $1,100$ & $\SciNot{2.26}{-3}$ & \bm{$\SciNot{1.68}{-11}$} \\
		$10$ &  $5$ & $20$ & $4,100$ & $\SciNot{2.34}{-1}$ & $\SciNot{1.16}{-7}$  \\
		
		$10$ & $10$ &  $5$ &   $450$ & $\SciNot{3.78}{-3}$ & \bm{$\SciNot{5.90}{-11}$} \\
		$10$ & $10$ & $10$ & $1,200$ & $\SciNot{6.18}{0}$  & $\SciNot{1.15}{-6}$  \\
		$10$ & $10$ & $20$ & $4,200$ & $\SciNot{2.71}{-2}$ & $\SciNot{3.69}{-8}$ 
	\end{tabular}\caption{Experiment~\ref{exp:3}: Results for differential problem \eqref{eq:pde-exp} obtained by different configurations of ExSpliNet with $p_\ell=q_\ell=3$, $L=2$.}  \label{tab:Ex3}
\end{table}
\begin{figure}[h!]
	\centering
	\subfigure{\includegraphics[trim={1cm 0.7cm 1cm 1cm}, clip, height=4.9cm]{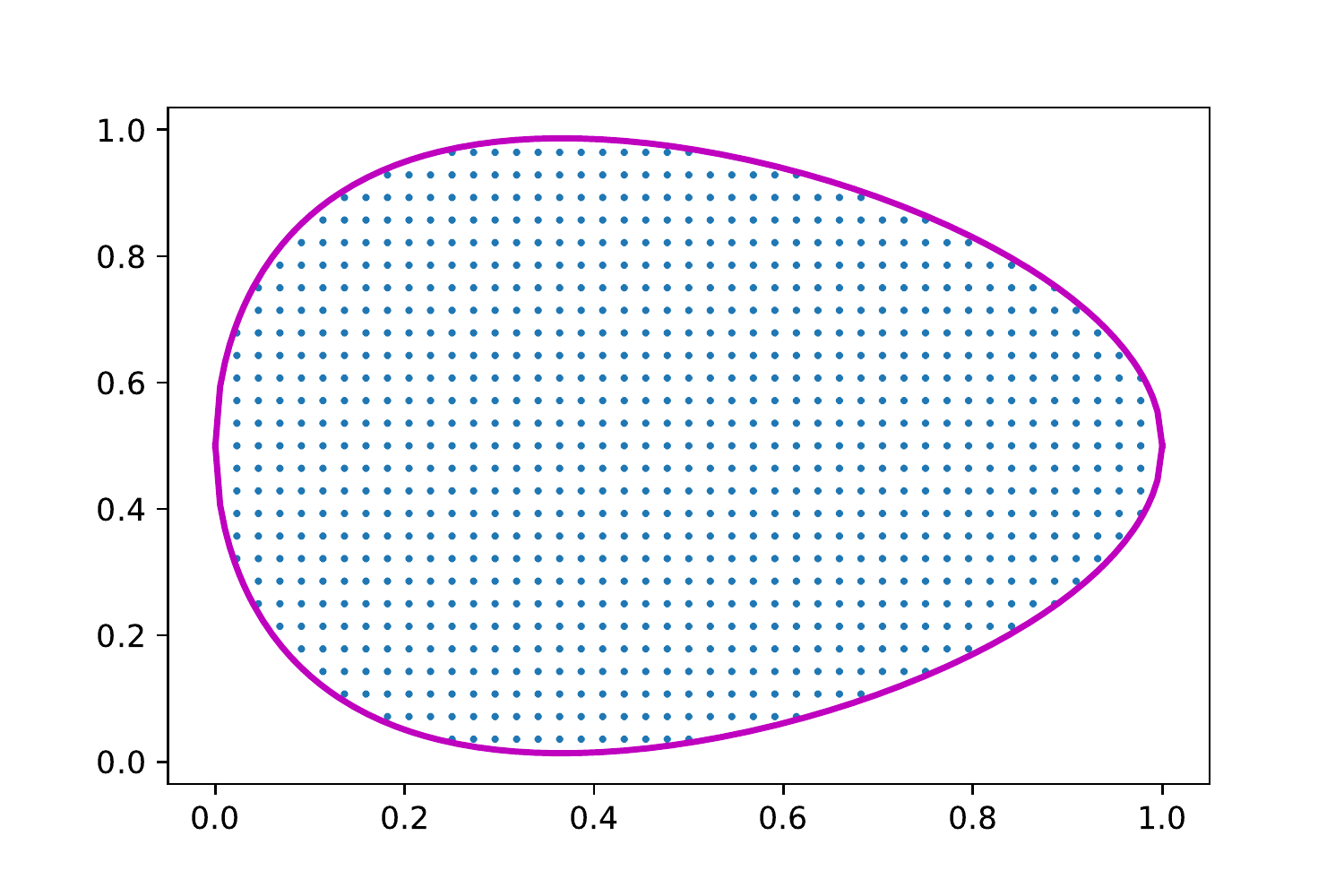}}
	\caption{The domain $\Omega$ of Experiment~\ref{exp:4}.}
	\label{fig:domain}
\end{figure}
\begin{experiment}\label{exp:4}
	We solved the differential equation \eqref{eq:pde} for $D=2$ and $\Omega$ is the egg shaped domain depicted in Figure~\ref{fig:domain}. The functions $f$ and $g$ are manufactured so that the exact solution of this problem is 
	$ u(x_1,x_2) = \sin (\pi (x_1^2 + x_2^2))$.
	For each run we took $2,062$ collocation points in the interior and $600$ on the boundary. The results obtained by different configurations of FFNN with tanh activation functions are summarized in Table~\ref{tab:NN4}, while the results obtained by different configurations of ExSpliNet with $p_\ell=q_\ell=3$ and $L=2$ are summarized in Table~\ref{tab:Ex4}. The three best results are indicated in boldface.
	In these tables, ``DER'' stands for the differential empirical risk after training, while ``MSE'' stands for the mean square error between the exact solution and the trained network model on $927$ uniformly spaced points in $\Omega$ (shown in Figure~\ref{fig:domain}).
\end{experiment}
\begin{table}[t!]
	\centering\small
	\begin{tabular}{rrrrr} 
		\hline
		$l$ & $n$ & Params & DER & MSE \\
		\hline & & & &\\[-1.9ex]
		 $4$ & $10$ &   $261$ & $\SciNot{2.06}{+0}$ & $\SciNot{1.31}{-4}$  \\
		 $4$ & $20$ &   $921$ & $\SciNot{9.81}{-1}$ & \bm{$\SciNot{5.02}{-5}$}  \\
		 $4$ & $30$ & $1,981$ & $\SciNot{1.06}{+0}$ & $\SciNot{5.26}{-5}$  \\
		 $6$ & $10$ &   $481$ & $\SciNot{1.71}{+0}$ & $\SciNot{1.69}{-4}$  \\
		 $6$ & $20$ & $1,761$ & $\SciNot{1.30}{+0}$ & \bm{$\SciNot{4.99}{-5}$}  \\
		 $6$ & $30$ & $3,841$ & $\SciNot{9.02}{-1}$ & \bm{$\SciNot{4.82}{-5}$}  \\
		 $8$ & $10$ &   $701$ & $\SciNot{2.68}{+0}$ & $\SciNot{2.58}{-4}$  \\
		 $8$ & $20$ & $2,601$ & $\SciNot{1.72}{+0}$ & $\SciNot{9.13}{-5}$  \\
		 $8$ & $30$ & $5,701$ & $\SciNot{4.07}{+0}$ & $\SciNot{9.42}{-5}$  \\
		$10$ & $10$ &   $921$ & $\SciNot{1.63}{+0}$ & $\SciNot{9.25}{-5}$  \\
		$10$ & $20$ & $3,441$ & $\SciNot{1.57}{+0}$ & $\SciNot{9.57}{-5}$  \\
		$10$ & $30$ & $7,561$ & $\SciNot{1.85}{+0}$ & $\SciNot{7.49}{-5}$   
	\end{tabular}\caption{Experiment~\ref{exp:4}: Results for the considered 2D differential problem obtained by different configurations of FFNN with tanh activation functions.}  \label{tab:NN4}
\end{table}
\begin{table}[t!]
	\centering\small
	\begin{tabular}{rrrrrr} \hline
		$T$ & $N_\ell$ & $M_\ell$ & Params & DER & MSE  \\ 
		\hline & & & & &\\[-1.9ex]
		 $5$ &  $5$ &  $5$ &   $225$ & $\SciNot{3.29}{-1}$ & $\SciNot{3.69}{-6}$  \\
		 $5$ &  $5$ & $10$ &   $600$ & $\SciNot{2.44}{-1}$ & $\SciNot{1.34}{-6}$  \\
		 $5$ &  $5$ & $20$ & $2,100$ & $\SciNot{1.50}{+0}$ & $\SciNot{2.07}{-6}$  \\
		 $5$ & $10$ &  $5$ &   $325$ & $\SciNot{2.34}{-1}$ & $\SciNot{1.03}{-6}$  \\
		 $5$ & $10$ & $10$ &   $700$ & $\SciNot{2.60}{-1}$ & $\SciNot{9.10}{-7}$  \\
		 $5$ & $10$ & $20$ & $2,200$ & $\SciNot{3.90}{-1}$ & $\SciNot{3.56}{-7}$  \\
		$10$ &  $5$ &  $5$ &   $450$ & $\SciNot{6.70}{-2}$ & $\SciNot{3.53}{-7}$  \\
		$10$ &  $5$ & $10$ & $1,200$ & $\SciNot{5.32}{-2}$ & \bm{$\SciNot{5.80}{-8}$}  \\
		$10$ &  $5$ & $20$ & $4,200$ & $\SciNot{7.39}{-2}$ & \bm{$\SciNot{1.18}{-7}$}  \\
		$10$ & $10$ &  $5$ &   $650$ & $\SciNot{1.83}{-1}$ & $\SciNot{1.76}{-7}$  \\
		$10$ & $10$ & $10$ & $1,400$ & $\SciNot{5.49}{-2}$ & \bm{$\SciNot{1.00}{-7}$}  \\
		$10$ & $10$ & $20$ & $4,400$ & $\SciNot{1.25}{-1}$ & {$\SciNot{2.04}{-7}$}  
	\end{tabular}\caption{Experiment~\ref{exp:4}: Results for the considered 2D differential problem obtained by different configurations of ExSpliNet with $p_\ell=q_\ell=3$, $L=2$.}  \label{tab:Ex4}
\end{table}

	\subsection{Classical learning tasks}
	In its current form, ExSpliNet is a general-purpose model, not designed nor tailored for a specific task (like convolutional neural networks for images or recurrent neural networks for text). In the following, we illustrate the general applicability of ExSpliNet in classical machine learning benchmark tasks, both classification and regression.

	\paragraph{Classification}  We first focus on the MNIST and FMNIST datasets for simple image classification \cite{LeCun,Xiao}.

\begin{experiment}\label{exp:MNIST}
	The MNIST dataset is a classical dataset of grayscale images representing handwritten digits. It contains $60,000$ training images and $10,000$ testing images. The task consists of classifying a given image of the dataset in one of the ten output classes, where each class represents a particular digit (so $O=10$). The images have $28\times28$ pixels (so $D=784$).
	Table~\ref{tab:MNIST} summarizes the results for different configurations of ExSpliNet, using $p_\ell=q_\ell=1$, $L=2$, and the Adam optimizer with a learning rate of $0.001$ and $35$ epochs. The best result is indicated in boldface. Here, ``\% Acc'' stands for the test accuracy of the classification in percentage. 
\end{experiment}
\begin{table}[t!]
	\centering\small
	\begin{tabular}{rrrrr} \hline
		$T$ & $N_\ell$ & $M_\ell$ & Params & \% Acc \\ \hline
		$50$ & $2$ & $2$ & $158,800$ & $96.25$ \\
		$50$ & $2$ & $3$ & $161,300$ & $97.91$ \\
		$50$ & $2$ & $4$ & $164,800$ & $98.26$ \\
		$50$ & $4$ & $2$ & $315,600$ & $96.59$ \\
		$50$ & $4$ & $3$ & $318,100$ & $97.87$ \\
		$50$ & $4$ & $4$ & $321,600$ & $97.94$ \\		
		$100$ & $2$ & $2$ & $317,600$ & $97.28$ \\
		$100$ & $2$ & $3$ & $322,600$ & \bm{$98.34$} \\
		$100$ & $2$ & $4$ & $329,600$ & $98.24$ \\
		$100$ & $4$ & $2$ & $631,200$ & $97.44$ \\
		$100$ & $4$ & $3$ & $636,200$ & $98.24$ \\
		$100$ & $4$ & $4$ & $643,200$ & $97.98$ \\
	\end{tabular}\caption{Experiment~\ref{exp:MNIST}: MNIST results obtained by different configurations of ExSpliNet with $p_\ell=q_\ell=1$, $L=2$.}  \label{tab:MNIST}
\end{table}

\begin{experiment}\label{exp:FMNIST}
	The FMNIST dataset is a classical dataset of grayscale images representing fashion products. It also contains $60,000$ training images and $10,000$ testing images. The task consists of classifying a given image of the dataset in one of the ten output classes (so $O=10$), where each class represents a particular garment. The images have $28\times28$ pixels (so $D=784$).
	Table~\ref{tab:FMNIST} summarizes the results for different configurations of ExSpliNet, using $p_\ell=q_\ell=1$, $L=2$, and the Adam optimizer with a learning rate of $0.001$ and $50$ epochs. The results are summarized in Table~\ref{tab:FMNIST}. The best result is indicated in boldface. Here, ``\% Acc'' stands for the test accuracy of the classification in percentage. 
\end{experiment}
\begin{table}[t!]
	\centering\small
	\begin{tabular}{rrrrr} \hline
		$T$ & $N_\ell$ & $M_\ell$ & Params & \% Acc \\ \hline
		$50$ & $2$ & $2$ & $158,800$ & $86.03$  \\
		$50$ & $2$ & $3$ & $161,300$ & $88.91$  \\
		$50$ & $2$ & $4$ & $164,800$ & $88.82$  \\
		$50$ & $4$ & $2$ & $315,600$ & $87.29$  \\
		$50$ & $4$ & $3$ & $318,100$ & $88.37$  \\
		$50$ & $4$ & $4$ & $321,600$ & $88.41$  \\		
		$100$ & $2$ & $2$ & $317,600$ & $86.69$  \\
		$100$ & $2$ & $3$ & $322,600$ & \bm{$89.56$}  \\
		$100$ & $2$ & $4$ & $329,600$ & $89.07$  \\
		$100$ & $4$ & $2$ & $631,200$ & $87.76$  \\
		$100$ & $4$ & $3$ & $636,200$ & $88.90$  \\
		$100$ & $4$ & $4$ & $643,200$ & $88.67$ \\
	\end{tabular}\caption{Experiment~\ref{exp:FMNIST}: FMNIST results obtained by different configurations of ExSpliNet with $p_\ell=q_\ell=1$, $L=2$.}  \label{tab:FMNIST}
\end{table}

	There exists a vast amount of literature covering the MNIST and FMNIST datasets, applying all kinds of different techniques, with or without preprocessing of the input data.
	As far as we know, the state-of-the-art results addressing the classification of MNIST images show a test accuracy of $99.91\%$; see, e.g., \cite{MNISTres}. We emphasize, however, that these results are obtained by means of ad hoc methods designed explicitly to execute such a task. Similarly, the state-of-the-art results addressing the classification of FMNIST images show a test accuracy of $96.91\%$; see, e.g., \cite{FMNISTres}.
	On the other hand, in its current form, ExSpliNet is a general-purpose model, not designed nor tailored for this specific task.
	To improve its performance in this context, the development and application of task-dependent inner functions for specific feature extraction might be an interesting direction of further research, for example, the investigation of convolutional spline operators such as in \cite{Fey}.

	\paragraph{Regression} 
	We now consider the Parkinson's telemonitoring dataset \cite{Tsanas2010a} for two regression tasks. %created by Tsanas and Little of the University of Oxford, in collaboration with 10 medical centers in the US and Intel Corporation.
	Remote tracking of Parkinson's desease (PD) progression is a medical practice that involves remotely monitoring patients who are not at the same location as the clinic. Usually, a patient has a monitoring device at home, and the resulting measurements are transmitted to the clinic through telephone or internet. Tracking PD symptom progression mostly relies on the unified Parkinson's disease rating scale (UPDRS), which displays presence and severity of symptoms \cite{Tsanas2010a}.
	The scale consists of 3 sections that assess (1) mentation, behavior, and mood, (2) activities of daily life, and (3) motor symptoms.  Total-UPDRS refers to the full range of the metric, 0--176, with 0 representing healthy and 176 total disability, and Motor-UPDRS refers to the motor section of the UPDRS ranging from 0 to 108.
	The Parkinson's telemonitoring dataset consists of a total of 5,875 voice recordings from 42 patients. Each voice recording consists of 16 biomedical voice measures (vocal features) and the related Motor-UPDRS and Total-UPDRS score.

	In the following two experiments we directly compare, in terms of mean absolute error (MAE), the performance of ExSpliNet with methods found in the literature.
	The Parkinson's telemonitoring problem was addressed in \cite{Tsanas2010a,Tsanas2010b} using linear regression least squares (LS), iteratively re-weighted least squares (IRLS), least absolute shrinkage and selection operator (LASSO), and a nonlinear regression method (CART).  	
	Other methods were used in \cite{Eskidere2012}:
	two kinds of support vector machines, regression SVM and LS-SVM, and two kinds of FFNN, namely MLPNN and GRNN.
	In those papers, to predict Motor-UPDRS and Total-UPDRS, it was observed that a log transformation of the vocal features reduces the test MAE, so we follow suit.

	In both ExSpliNet experiments we used the Adam optimizer with a learning rate of $0.001$ and $100$ epochs.
	As in \cite{Eskidere2012}, ten-fold cross validation was applied for evaluating the test performance of the model by averaging the MAE results obtained from all folds. For each considered setup, the averaged MAE value is provided and its standard deviation is marked in round brackets.

\begin{experiment}\label{exp:Motor-UPDRS}
	For the prediction of Motor-UPDRS, ExSpliNet (with $T = 100$, $L = 2$, $N_l = 2$, $M_l =25$, $p_l = q_l = 1$) achieves an MAE of \bm{$4.67$} \bm{$(0.12)$} obtained on the log transformed features (before normalization in $[0,1]$).
	On the other hand, in \cite{Eskidere2012} it is stated that the LS-SVM outperformed the other proposed methods for the prediction of Motor-UPDRS, in terms of lower prediction errors, with a best MAE of $4.87$ $(0.11)$ obtained on the log transformed features, beating the best results in \cite{Tsanas2010a,Tsanas2010b}, i.e., an MAE of $5.95$ $(0.19)$ obtained by CART and an MAE of $6.57$ $(0.17)$ obtained by LASSO.
\end{experiment}
\begin{experiment}\label{exp:Total-UPDRS}
    For the prediction of Total-UPDRS, ExSpliNet (with $T = 100$, $L = 2$, $N_l = 2$, $M_l =25$, $p_l = q_l = 1$) achieves an MAE of \bm{$5.95$} \bm{$(0.19)$} obtained on the log transformed features (before normalization in $[0,1]$).
    On the other hand, in \cite{Eskidere2012} it is stated that the LS-SVM outperformed the other proposed methods for the prediction of Total-UPDRS, in terms of lower prediction errors, with a best MAE of
	$6.18$ $(0.16)$ obtained on the log transformed features, beating the best results in \cite{Tsanas2010a,Tsanas2010b}, i.e., an MAE of $7.52$ $(0.25)$ obtained by CART and an MAE of $8.38$ $(0.23)$ obtained by LASSO.
\end{experiment}

	For the sake of completeness, we also mention that more recent results are found in \cite{Nilashi2016}. Therein, instead of applying directly regression methods, a more involved feature engineering procedure was performed: the dataset was subdivided in clusters using the EM algorithm and for each cluster a subset of features was selected and different methods were then performed, resulting in a lower MAE.

	\section{Conclusion}\label{sec:Conclusions}
	In this paper we proposed an interpretable and expressive neural network model, called ExSpliNet.
	We started with a definition of ExSpliNet as a Kolmogorov-like neural network model with univariate spline inner functions and $L$-variate spline outer functions, all of them in B-spline representations. Furthermore, we detailed a probabilistic interpretation in terms of feature extractors and probabilistic trees. In addition, we showed universal approximation results in the cases $L=1$ and $L=D$. We also discussed how the model can be efficiently encoded by exploiting B-spline properties.
	Finally, we tested the performance of the model on a small selection of synthetic approximation problems and classical machine learning benchmark datasets.

	ExSpliNet is highly customizable as it is steered by several hyperparameters. The presented experiments indicate that ExSpliNet outperforms FFNN in approximation tasks. The choice of high degrees seems to directly pay off in the accuracy for smooth functions. This is a feature inherited from tensor-product splines. ExSpliNet is particularly suited for solving differential problems in the PINN framework, as it combines high smoothness with good approximation properties. However, a more extended experimental work, targeting nonlinear and high-dimensional differential problems, is required to reveal the full potential of ExSpliNet in this field. In addition, a more profound study of the approximation capabilities is essential for a better understanding of the model and its performance, especially for the intermediate cases $1<L<D$ (see Remark~\ref{rmk:L-intermediate}). 

	In its current form, ExSpliNet is a general-purpose model.
	Since the spline additive model is very general, the development of task-dependent inner functions for specific feature extraction could improve the performance of ExSpliNet in the context of classical machine learning. An interesting direction is the incorporation of convolutional spline operators such as in \cite{Fey}.
	The study of the statistic learning theoretical generalization ability of the model is another promising direction of investigation.
	The exploration of nonuniform knot sequences, with possibly coinciding interior knots, could also be helpful to further increase the flexibility of the model.

	\section*{Acknowledgments}
	This work was supported in part by the MIUR Excellence Department Project awarded to the Department of Mathematics, University of Rome Tor Vergata (CUP E83C18000100006).

%%%%%%%%%%%%%%
% References %
%%%%%%%%%%%%%%

\bibliographystyle{plain}
\bibliography{exsplinet}

\begin{thebibliography}{10}

\bibitem{Agarwal}
R.~Agarwal, N.~Frosst, X.~Zhang, R.~Caruana, and G.~E. Hinton.
\newblock Neural additive models: Interpretable machine learning with neural
  nets.
\newblock In U.~Bhatt, A.~Dhurandhar, B.~Kim, K.~R. Varshney, D.~Wei,
  A.~Weller, and A.~Xiang, editors, {\em Proceedings of the 2020 {ICML}
  Workshop on Human Interpretability in Machine Learning}, 2020.

\bibitem{Alkhoury}
S.~Alkhoury, E.~Devijver, M.~Clausel, M.~Tami, E.~Gaussier, and G.~Oppenheim.
\newblock Smooth and consistent probabilistic regression trees.
\newblock In H.~Larochelle, M.~Ranzato, R.~Hadsell, M.~F. Balcan, and H.~Lin,
  editors, {\em Advances in Neural Information Processing Systems}, volume~33,
  pages 11345--11355. Curran Associates, 2020.

\bibitem{MNISTres}
S.~An, M.~Lee, S.~Park, H.~Yang, and J.~So.
\newblock An ensemble of simple convolutional neural network models for {MNIST}
  digit recognition.
\newblock {\em ArXiv preprint 2008.10400}, 2020.

\bibitem{Andras}
P.~Andras.
\newblock Function approximation using combined unsupervised and supervised
  learning.
\newblock {\em IEEE Transactions on Neural Networks and Learning Systems},
  25:495--505, 2014.

\bibitem{Bach}
F.~Bach.
\newblock Breaking the curse of dimensionality with convex neural networks.
\newblock {\em Journal of Machine Learning Research}, 18:1--53, 2017.

\bibitem{Balestriero1}
R.~Balestriero.
\newblock Neural decision trees.
\newblock {\em ArXiv preprint 1702.07360}, 2017.

\bibitem{Balestriero2}
R.~Balestriero and R.~G. Baraniuk.
\newblock Mad max: Affine spline insights into deep learning.
\newblock {\em Proceedings of the IEEE}, 109:704--727, 2021.

\bibitem{Bellomarini}
L.~Bellomarini, D.~Fakhoury, G.~Gottlob, and E.~Sallinger.
\newblock Knowledge graphs and enterprise {AI}: The promise of an enabling
  technology.
\newblock In {\em 2019 {IEEE} 35th International Conference on Data
  Engineering}, pages 26--37. IEEE, 2019.

\bibitem{Biau}
G.~Biau, E.~Scornet, and J.~Welbl.
\newblock Neural random forests.
\newblock {\em Sankhya A}, 81:347--386, 2019.

\bibitem{Bohra}
P.~Bohra, J.~Campos, H.~Gupta, S.~Aziznejad, and M.~Unser.
\newblock Learning activation functions in deep (spline) neural networks.
\newblock {\em IEEE Open Journal of Signal Processing}, 1:2020, 295--309.

\bibitem{Braun}
J.~Braun and M.~Griebel.
\newblock On a constructive proof of {K}olmogorov's superposition theorem.
\newblock {\em Constructive Approximation}, 30:653--675, 2009.

\bibitem{Breiman}
L.~Breiman, J.~H. Friedman, R.~A. Olshen, and C.~J. Stone.
\newblock {\em Classification and Regression Trees}.
\newblock Brooks/Cole Publishing, 1984.

\bibitem{Bressan}
A.~Bressan and E.~Sande.
\newblock Approximation in {FEM}, {DG} and {IGA}: A theoretical comparison.
\newblock {\em Numerische Mathematik}, 143:923--942, 2019.

\bibitem{Campolucci}
P.~Campolucci, F.~Capperelli, S.~Guarnieri, F.~Piazza, and A.~Uncini.
\newblock Neural networks with adaptive spline activation function.
\newblock In M.~De~Sario, B.~Maione, P.~Pugliese, and M.~Savino, editors, {\em
  Proceedings of 8th Mediterranean Electrotechnical Conference on Industrial
  Applications in Power Systems, Computer Science and Telecommunications},
  volume~3, pages 1442--1445, 1996.

\bibitem{Coelho2}
L.~D.~S. Coelho and M.~W. Pess\^oa.
\newblock Nonlinear identification using a {B}-spline neural network and
  chaotic immune approaches.
\newblock {\em Mechanical Systems and Signal Processing}, 23:2418--2434, 2009.

\bibitem{Cohen}
N.~Cohen, O.~Sharir, and A.~Shashua.
\newblock On the expressive power of deep learning: A tensor analysis.
\newblock In V.~Feldman, A.~Rakhlin, and O.~Shamir, editors, {\em 29th Annual
  Conference on Learning Theory}, volume~49 of {\em Proceedings of Machine
  Learning Research}, pages 698--728. PMLR, 2016.

\bibitem{Costarelli}
D.~Costarelli and R.~Spigler.
\newblock Approximation by series of sigmoidal functions with applications to
  neural networks.
\newblock {\em Annali di Matematica Pura ed Applicata}, 194:289--306, 2015.

\bibitem{Cybenko}
G.~Cybenko.
\newblock Approximation by superpositions of a sigmoidal function.
\newblock {\em Mathematics of Control, Signals and Systems}, 2:303--314, 1989.

\bibitem{deBoor}
C.~de~Boor.
\newblock {\em A Practical Guide to Splines}.
\newblock Springer--Verlag, New York, revised edition, 2001.

\bibitem{Eskidere2012}
\"{O}. Eskidere, F.~Erta{\c{s}}, and C.~Hanil{\c{c}}i.
\newblock A comparison of regression methods for remote tracking of
  {P}arkinson's disease progression.
\newblock {\em Expert Systems with Applications}, 39:5523--5528, 2012.

\bibitem{Fey}
M.~Fey, J.~Lenssen, F.~Weichert, and H.~Muller.
\newblock Spline{CNN}: Fast geometric deep learning with continuous {B}-spline
  kernels.
\newblock In {\em 2018 IEEE/CVF Conference on Computer Vision and Pattern
  Recognition}, pages 869--877. IEEE Computer Society, 2018.

\bibitem{Fisher}
R.~A. Fisher.
\newblock The use of multiple measurements in taxonomic problems.
\newblock {\em Annual Eugenics}, 7:179--188, 1936.

\bibitem{Friedman}
J.~H. Friedman.
\newblock Adaptive spline networks.
\newblock In R.~P. Lippmann, J.~Moody, and D.~Touretzky, editors, {\em Advances
  in Neural Information Processing Systems}, volume~3, pages 675--683.
  Morgan-Kaufmann, 1991.

\bibitem{Guarnieri}
S.~Guarnieri, F.~Piazza, and A.~Uncini.
\newblock Multilayer feedforward networks with adaptive spline activation
  function.
\newblock {\em IEEE Transactions on Neural Networks}, 10:672--683, 1999.

\bibitem{Harris}
C.~J. Harris, C.~G. Moore, and M.~Brown.
\newblock The {B}-spline neural network and fuzzy logic.
\newblock In {\em Intelligent Control: Aspects of Fuzzy Logic and Neural Nets},
  pages 314--357. World Scientific Press, 1993.

\bibitem{Hastie2}
T.~Hastie, R.~Tibshirani, and J.~Friedman.
\newblock {\em The Elements of Statistical Learning: Data Mining, Inference,
  and Prediction}.
\newblock Springer, 2nd edition, 2009.

\bibitem{Hastie}
T.~J. Hastie and R.~J. Tibshirani.
\newblock {\em Generalized Additive Models}.
\newblock Chapman and Hall, 1990.

\bibitem{Hecht}
R.~Hecht-Nielsen.
\newblock Kolmogorov's mapping neural network existence theorem.
\newblock In {\em Proceedings of the IEEE First International Conference on
  Neural Networks}, volume~3, pages 11--13. IEEE Press, 1987.

\bibitem{Hornik}
K.~Hornik, M.~Stinchcombe, and H.~White.
\newblock Multilayer feedforward networks are universal approximators.
\newblock {\em Neural Networks}, 2:359--366, 1989.

\bibitem{Igelnik}
B.~Igelnik and N.~Parikh.
\newblock Kolmogorov's spline network.
\newblock {\em IEEE Transactions on Neural Networks}, 14:725--733, 2003.

\bibitem{Karagoz}
R.~Karagoz and K.~Batselier.
\newblock Nonlinear system identification with regularized {T}ensor {N}etwork
  {B}-splines.
\newblock {\em Automatica}, 122:109300, 2020.

\bibitem{Kingma}
D.~P. Kingma and L.~J. Ba.
\newblock Adam: A method for stochastic optimization.
\newblock In Y.~Bengio and Y.~LeCun, editors, {\em Proceedings of the 3rd
  International Conference on Learning Representations}, 2015.

\bibitem{Kolmogorov}
A.~N. Kolmogorov.
\newblock On the representation of continuous functions of several variables by
  superposition of continuous functions of one variable and addition.
\newblock {\em Doklady Akademii Nauk SSSR}, 114:953--956, 1957.

\bibitem{Kontschieder}
P.~Kontschieder, M.~Fiterau, A.~Criminisi, and S.~R. Bul\`o.
\newblock Deep neural decision forests.
\newblock In {\em Proceedings of the {IEEE} International Conference on
  Computer Vision}, pages 1467--1475. IEEE, 2015.

\bibitem{Koppen}
M.~K\"oppen.
\newblock On the training of a {K}olmogorov network.
\newblock In J.~R. Dorronsoro, editor, {\em Artificial Neural Networks - ICANN
  2002}, volume 2415 of {\em Lecture Notes in Computer Science}, pages
  474--479. Springer, 2002.

\bibitem{Kurkova1}
V.~K\r{u}rkov\'a.
\newblock Kolmogorov's theorem is relevant.
\newblock {\em Neural Computation}, 3:617--622, 1991.

\bibitem{Kurkova2}
V.~K\r{u}rkov\'a.
\newblock Kolmogorov's theorem and multilayer neural networks.
\newblock {\em Neural Networks}, 5:501--506, 1992.

\bibitem{LeCun}
Y.~LeCun, L.~Bottou, Y.~Bengio, and P.~Haffner.
\newblock Gradient-based learning applied to document recognition.
\newblock {\em Proceedings of the IEEE}, 86:2278--2324, 1998.

\bibitem{Lightbody}
G.~Lightbody, P.~O'Reilly, G.~W. Irwin, K.~Kelly, and J.~McCormick.
\newblock Neural modelling of chemical plant using {MLP} and {B}-spline
  networks.
\newblock {\em Control Engineering Practice}, 5:1501--1515, 1997.

\bibitem{Lu}
L.~Lu, X.~Meng, Z.~Mao, and G.~E. Karniadakis.
\newblock {DeepXDE}: A deep learning library for solving differential
  equations.
\newblock {\em SIAM Review}, 63:208--228, 2021.

\bibitem{Lyche}
T.~Lyche, C.~Manni, and H.~Speleers.
\newblock Foundations of spline theory: B-splines, spline approximation, and
  hierarchical refinement.
\newblock In T.~Lyche, C.~Manni, and H.~Speleers, editors, {\em Splines and
  PDEs: From Approximation Theory to Numerical Linear Algebra}, volume 2219 of
  {\em Lecture Notes in Mathematics}, pages 1--76. Springer International
  Publishing, 2018.

\bibitem{Marsden}
M.~Marsden.
\newblock An identity for spline functions and its application to variation
  diminishing spline approximation.
\newblock {\em Journal of Approximation Theory}, 3:7--49, 1970.

\bibitem{Montanelli1}
H.~Montanelli and Q.~Du.
\newblock New error bounds for deep {ReLU} networks using sparse grids.
\newblock {\em SIAM Journal on Mathematics of Data Science}, 1:78--92, 2019.

\bibitem{Montanelli2}
H.~Montanelli and H.~Yang.
\newblock Error bounds for deep {ReLU} networks using the {K}olmogorov-{A}rnold
  superposition theorem.
\newblock {\em Neural Networks}, 129:1--6, 2020.

\bibitem{Nilashi2016}
M.~Nilashi, O.~Ibrahim, and A.~Ahani.
\newblock Accuracy improvement for predicting {P}arkinson's disease
  progression.
\newblock {\em Scientific Reports}, 6:34181, 2016.

\bibitem{Poggio}
T.~Poggio, H.~Mhaskar, L.~Rosasco, B.~Miranda, and Q.~Liao.
\newblock Why and when can deep-but not shallow-networks avoid the curse of
  dimensionality: A review.
\newblock {\em International Journal of Automation and Computing}, 14:503--519,
  2017.

\bibitem{Potts}
W.~J.~E. Potts.
\newblock Generalized additive neural networks.
\newblock In S.~Chaudhuri and D.~Madigan, editors, {\em Proceedings of the
  Fifth {ACM SIGKDD} International Conference on Knowledge Discovery and Data
  Mining}, pages 194--200. ACM, 1999.

\bibitem{Raissi}
M.~Raissi, P.~Perdikaris, and G.~E. Karniadakis.
\newblock Physics-informed neural networks: A deep learning framework for
  solving forward and inverse problems involving nonlinear partial differential
  equations.
\newblock {\em Journal of Computational Physics}, 378:686--707, 2019.

\bibitem{Ribeiro}
M.~Ribeiro, S.~Singh, and C.~Guestrin.
\newblock ``{W}hy should {I} trust you'': Explaining the predictions of any
  classifier.
\newblock In J.~DeNero, M.~Finlayson, and S.~Reddy, editors, {\em Proceedings
  of the 2016 Conference of the North American Chapter of the Association for
  Computational Linguistics: Demonstrations}, pages 97--101. ACL, 2016.

\bibitem{Sande1}
E.~Sande, C.~Manni, and H.~Speleers.
\newblock Sharp error estimates for spline approximation: Explicit constants,
  $n$-widths, and eigenfunction convergence.
\newblock {\em Mathematical Models and Methods in Applied Sciences},
  29:1175--1205, 2019.

\bibitem{Sande2}
E.~Sande, C.~Manni, and H.~Speleers.
\newblock Explicit error estimates for spline approximation of arbitrary
  smoothness in isogeometric analysis.
\newblock {\em Numerische Mathematik}, 144:889--929, 2020.

\bibitem{Scardapane}
S.~Scardapane, M.~Scarpiniti, D.~Comminiello, and A.~Uncini.
\newblock Learning activation functions from data using cubic spline
  interpolation.
\newblock In A.~Esposito, M.~Faundez-Zanuy, F.~C. Morabito, and E.~Pasero,
  editors, {\em Neural Advances in Processing Nonlinear Dynamic Signals},
  volume 102 of {\em Smart Innovation, Systems and Technologies}, pages 73--83.
  Springer, 2019.

\bibitem{Schumaker}
L.~L. Schumaker.
\newblock {\em Spline Functions: Basic Theory}.
\newblock Cambridge University Press, 3rd edition, 2007.

\bibitem{Selvaraju}
R.~R. Selvaraju, M.~Cogswell, A.~Das, R.~Vedantam, D.~Parikh, and D.~Batra.
\newblock Grad-{CAM}: Visual explanations from deep networks via gradient-based
  localization.
\newblock {\em International Journal of Computer Vision}, 128:336--359, 2020.

\bibitem{Sprecher1}
D.~A. Sprecher.
\newblock A numerical implementation of {K}olmogorov's superpositions.
\newblock {\em Neural Networks}, 9:765--772, 1996.

\bibitem{Sprecher2}
D.~A. Sprecher.
\newblock A numerical implementation of {K}olmogorov's superpositions {II}.
\newblock {\em Neural Networks}, 10:447--457, 1997.

\bibitem{Srivastava}
N.~Srivastava, G.~Hinton, A.~Krizhevsky, I.~Sutskever, and R.~Salakhutdinov.
\newblock Dropout: A simple way to prevent neural networks from overfitting.
\newblock {\em Journal of Machine Learning Research}, 15:1929--1958, 2014.

\bibitem{FMNISTres}
M.~Tanveer, M.~U.~K. Khan, and C.-M. Kyung.
\newblock Fine-tuning {DARTS} for image classification.
\newblock In {\em 2020 25th International Conference on Pattern Recognition},
  pages 4789--4796. IEEE Computer Society, 2021.

\bibitem{Telgarsky}
M.~Telgarsky.
\newblock Benefits of depth in neural networks.
\newblock In V.~Feldman, A.~Rakhlin, and O.~Shamir, editors, {\em 29th Annual
  Conference on Learning Theory}, volume~49 of {\em Proceedings of Machine
  Learning Research}, pages 1517--1539. PMLR, 2016.

\bibitem{Tsanas2010a}
A.~Tsanas, M.~A. Little, P.~E. McSharry, and L.~O. Ramig.
\newblock Accurate telemonitoring of {P}arkinson's disease progression by
  noninvasive speech tests.
\newblock {\em IEEE Transactions on Biomedical Engineering}, 57(4):884--893,
  2010.

\bibitem{Tsanas2010b}
A.~Tsanas, M.~A. Little, P.~E. McSharry, and L.~O. Ramig.
\newblock Enhanced classical dysphonia measures and sparse regression for
  telemonitoring of {P}arkinson's disease progression.
\newblock In {\em 2010 {IEEE} International Conference on Acoustics, Speech and
  Signal Processing}, pages 594--597. {IEEE}, 2010.

\bibitem{Vecci}
L.~Vecci, F.~Piazza, and A.~Uncini.
\newblock Learning and approximation capabilities of adaptive spline activation
  function neural networks.
\newblock {\em Neural Networks}, 11:259--270, 1998.

\bibitem{WangL}
K.~Wang and B.~Lei.
\newblock Using {B}-spline neural network to extract fuzzy rules for a
  centrifugal pump monitoring.
\newblock {\em Journal of Intelligent Manufacturing}, 12:5--11, 2001.

\bibitem{Wang}
S.~Wang, C.~Aggarwal, and H.~Liu.
\newblock Using a random forest to inspire a neural network and improving on
  it.
\newblock In N.~Chawla and W.~Wang, editors, {\em Proceedings of the 2017
  {SIAM} International Conference on Data Mining}, pages 1--9. SIAM, 2017.

\bibitem{Xiao}
H.~Xiao, K.~Rasul, and R.~Vollgraf.
\newblock {F}ashion-{MNIST}: A novel image dataset for benchmarking machine
  learning algorithms.
\newblock {\em ArXiv preprint 1708.07747}, 2017.

\bibitem{Yang}
Y.~Yang, I.~G. Morillo, and T.~M. Hospedales.
\newblock Deep neural decision trees.
\newblock In B.~Kim, K.~R. Varshney, and A.~Weller, editors, {\em Proceedings
  of the 2018 {ICML} Workshop on Human Interpretability in Machine Learning},
  pages 34--40, 2018.

\bibitem{Yarotsky2}
D.~Yarotsky.
\newblock Optimal approximation of continuous functions by very deep {ReLU}
  networks.
\newblock In S.~Bubeck, V.~Perchet, and P.~Rigollet, editors, {\em Proceedings
  of the 31st Conference On Learning Theory}, volume~75 of {\em Proceedings of
  Machine Learning Research}, pages 639--649. PMLR, 2018.

\end{thebibliography}

\end{document}